\documentclass{article}
\PassOptionsToPackage{table}{xcolor}

\PassOptionsToPackage{numbers,compress}{natbib}
\usepackage[preprint]{template/neurips_2026}
\usepackage{hyperref}
\usepackage{times}
\usepackage{latexsym}
\usepackage{amsmath}
\usepackage{amssymb}
\usepackage{amsthm}
\usepackage{graphicx}
\usepackage{wrapfig}
\usepackage{booktabs}
\usepackage{multirow}
\usepackage{enumitem}
\usepackage{xcolor}
\usepackage[most]{tcolorbox}
\tcbuselibrary{skins}
\usepackage{subcaption}
\usepackage{pifont}
\usepackage{float}
\usepackage{algorithm}
\usepackage{algpseudocode}
\DeclareMathOperator*{\argmin}{arg\,min}
\usepackage{tabularx}
\usepackage{booktabs}
\newcolumntype{L}{>{\raggedright\arraybackslash}X}
\newcolumntype{S}{>{\raggedright\arraybackslash\hsize=0.75\hsize}X}
\newcolumntype{B}{>{\raggedright\arraybackslash\hsize=1.0833\hsize}X}
\newcolumntype{N}{>{\raggedright\arraybackslash\hsize=0.6\hsize}X}
\newcolumntype{W}{>{\raggedright\arraybackslash\hsize=1.1333\hsize}X}
\newtheorem{proposition}{Proposition}
\newtheorem*{proposition*}{Proposition}

\newcommand{\cmark}{\ding{51}}
\newcommand{\xmark}{\ding{55}}
\AtBeginDocument{%
  \setlength{\abovedisplayskip}{4pt}%
  \setlength{\belowdisplayskip}{4pt}%
  \setlength{\abovedisplayshortskip}{2pt}%
  \setlength{\belowdisplayshortskip}{2pt}%
}

\definecolor{PTGreen}{RGB}{92,134,110}
\definecolor{PTGreenDark}{RGB}{52,92,75}
\definecolor{PTLightGreen}{RGB}{242,250,245}
\definecolor{PTWhite}{RGB}{255,255,255}

\title{Same Signal, Opposite Meaning: Direction-Informed Adaptive Learning for LLM Agents}

\author{%
    \textbf{Ziming Li\textsuperscript{1}},
    \textbf{Jiatan Huang\textsuperscript{1}},
    \textbf{Xiaoguang Guo\textsuperscript{1}},
    \textbf{Guiling Wang\textsuperscript{2}},
    \textbf{Chuxu Zhang\textsuperscript{1}\thanks{Corresponding author.}}\\
    \small\texttt{\{ziming.li,jiatan.huang,xiaoguang.guo,chuxu.zhang\}@uconn.edu} \\
    \small\texttt{gwang@njit.edu} \\
    \textsuperscript{1}{University of Connecticut},
    \textsuperscript{2}{New Jersey Institute of Technology}
}

\begin{document}
\maketitle

% ============================================================
% ABSTRACT
% ============================================================
\vspace{-0.1in}
\begin{abstract}
\vspace{-0.1in}
Adaptive test-time compute for LLM agents aims to invoke extra computation only when it improves performance. Existing methods typically use confidence-, uncertainty-, or difficulty-based gates, assuming a fixed direction from the gating signal through compute need to the value of computation. This makes gating a utility-calibration problem: gating signals should align with whether extra computation improves the final outcome over the base policy. We show that this alignment is unstable: the same signal predicts rollout benefit in one setting and rollout harm in another, with reversals across environments and backbones even when the task is fixed. Wrong-direction gates can therefore worsen performance by precisely selecting harmful states.
This reversal reflects a deeper distinction between compute need and compute suitability: a high uncertainty signal may indicate decision-difficult states where rollouts help compare alternatives, or intervention-unsuitable states where the current context does not support useful rollout-based improvement. Under this two-source model, fixed-direction gates are unreliable across heterogeneous settings. To address this, we propose DIAL (\textbf{D}irection-\textbf{I}nformed \textbf{A}daptive \textbf{L}earning), a sparse gate trained from signal-agnostic counterfactual exploration to learn the utility direction of state features per (environment, backbone). Across six environments and three backbones, DIAL yields a stronger overall success-cost trade-off than fixed-direction baselines. Our code is open-sourced at \url{https://github.com/zmli6/DIAL}.
\end{abstract}

% ============================================================
% 1. INTRODUCTION
% ============================================================
\section{Introduction}
\label{sec:intro}
\vspace{-0.1in}
% --- Module 1: Why this problem matters (~45%) ---

Large language models (LLMs) have rapidly advanced across a wide range of tasks, from question answering and code generation to embodied decision making~\cite{yao2023react,
shinn2023reflexion, zhou2024lats, li2025mass, li2026graph}. However, a single forward pass is often insufficient for complex tasks that require evaluating alternatives, verifying intermediate results, or searching over candidate actions~\cite{yao2023tree, hao2023reasoning, snell2024scaling}. This has motivated test-time compute, where additional computation at inference time improves decision quality~\cite{wang2026flare,DBLP:journals/corr/abs-2506-12928}.
However, test-time compute is expensive: each decision point may require multiple rollouts or candidate evaluations~\cite{zhou2024lats,yao2023tree}. Applying it uniformly across diverse agent settings is wasteful, since only a fraction of steps benefit. This has motivated \emph{adaptive gating}: using uncertainty signals to decide \emph{when} extra compute is worthwhile. Prior work explores various mechanisms, including entropy thresholding~\cite{seag2025, cats2025, corefine2026}, calibrated confidence~\cite{s1_2025, han2025tokenbudget}, vote disagreement~\cite{catts2026}, and learned difficulty estimators based on hidden states or reinforcement learning~\cite{diffadapt2025, zhu2025llmalreadyknows, adaptthink2025, thinkless2025}. These methods differ in mechanism but share a common assumption: larger uncertainty, disagreement, or difficulty signals indicate greater value of additional computation. We call this the \emph{fixed-direction} assumption. In interactive agents, this makes gating a utility-calibration problem: the signal must predict the value of computation, namely whether invoking the optimizer improves the final task outcome. This differs from standard confidence calibration, which aligns confidence with action correctness.

We test this assumption empirically using token entropy, the most common signal in prior work~\cite{seag2025,cats2025,corefine2026}. The fixed-direction assumption predicts $\rho(\text{entropy}, U) > 0$: higher entropy should mark states where rollout has positive value. Figure~\ref{fig:intro_reversal} finds this violated on half of six agent environments: high entropy can predict that extra compute hurts rather than helps. Table~\ref{tab:signal-discovery} further shows that the same task can flip sign across model backbones, so a signal's utility semantics are determined by the (environment, backbone) pair, not by the environment or backbone alone. This reversal is consequential: when the assumed direction is wrong, sharper gates worsen performance by more precisely selecting harmful states. The most predictive signal also varies across environments (Table~\ref{tab:signal-identity}), so adaptive compute cannot rely on a fixed signal with a fixed direction. This raises a central question: \emph{if signal-utility direction is not fixed, can adaptive compute still be principled?}

\begin{wrapfigure}{r}{0.48\textwidth}
	\vspace{-12pt}
	\centering
	\includegraphics[width=0.46\textwidth]{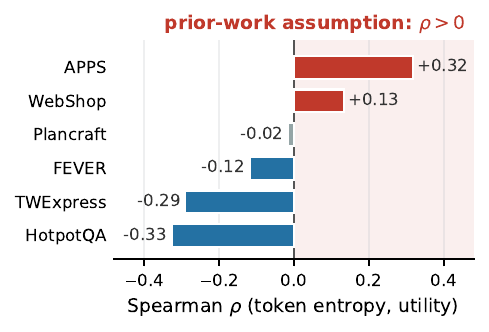}
	\vspace{-0.1in}
	\caption{\textbf{The fixed-direction assumption fails on half
			of the environments.} Spearman correlation
		$\rho(\text{entropy}, U)$ across six environments on
		Qwen3-4B.}
	\label{fig:intro_reversal}
	\vspace{-10pt}
\end{wrapfigure}

Answering this requires understanding why direction reverses in the first place. Hidden in the fixed-direction assumption is a deeper conflation: uncertainty or difficulty measures whether a state may benefit from more compute, but not whether that compute will actually help.
We trace this reversal to two coexisting state types with opposite signal-utility relationships. \emph{Intervention-unsuitable} states (Type~I) arise when the current context does not support useful rollout-based improvement (e.g., missing evidence, misleading partial observations, or weak grounding); high uncertainty there predicts negative value of computation, since additional rollouts amplify unresolved ambiguity rather than resolve it. \emph{Decision-difficult} states (Type~D, e.g., coding tasks with multiple viable implementations) arise when multiple options exist; rollouts can sample different ones and find the best, so high uncertainty predicts positive value. The mixture depends on what the agent already encodes, so it is determined by the (environment, backbone) pair; the same uncertainty signal therefore carries opposite meanings depending on which type dominates, an instance of Simpson's paradox~\cite{simpson1951interpretation, kdd2023simpson}: the marginal signal-utility correlation reflects the mixture of two subgroups whose internal correlations point in opposite directions.

%The two-source decomposition explains why uncertainty-only gates fail. Two states with the same uncertainty $\sigma$ can have opposite-sign utility if they belong to different types, so no function of $\sigma$ alone can tell them apart (Proposition~\ref{prop:necessity}). The decomposition also predicts how $\rho(\sigma, U)$ depends on the Type~I/Type~D mixture, which we verify on held-out data (\S\ref{sec:theory-verification}). A reliable gate must therefore augment $\sigma$ with features that correlate with state type, such as step count as a proxy for information accumulation.
This is why uncertainty alone cannot reliably gate: two states at the same uncertainty level can have opposite utility, so no function of uncertainty alone can distinguish them. A reliable gate must therefore use additional features that correlate with state type, such as step count as a proxy for information accumulation. This also clarifies why replacing thresholds with learned difficulty estimators or confidence calibration does not automatically resolve the problem. Calibration aligns confidence with action correctness, whereas adaptive gating requires alignment with the counterfactual utility of a particular optimizer. Difficulty estimators may identify that a state is hard, but hardness can mean either that rollouts can productively compare alternatives or that the current context is unsuitable for improvement. Moreover, a learned gate trained only on uncertainty-triggered states can inherit the same fixed-direction assumption through selection bias, baking the prior into the training distribution rather than testing it. The decomposition further yields testable predictions about how $\rho(\sigma, U)$ changes with the Type~I/Type~D mixture, which we verify on held-out data in \S\ref{sec:theory-verification}.
This motivates data-driven direction learning: utility depends on interaction outcomes and cannot generally be inferred from model-internal properties alone. Building on this analysis, we propose DIAL (\textbf{D}irection-\textbf{I}nformed \textbf{A}daptive \textbf{L}earning), which learns the signal-utility direction for each (environment, backbone) setting from data. DIAL collects signal-agnostic paired counterfactual outcomes by running the optimizer and the base policy from the same state, then fits a sparse linear gate over a small set of state features augmented by LLM-proposed task-specific signals. Across six environments and three LLM backbones, DIAL consistently improves the success-cost trade-off over fixed-direction baselines, triggering frequently when rollouts help and rarely when they do not.
%Our main contributions are as follows:
%\begin{itemize}[leftmargin=*,nosep]
%\item \textbf{Empirical finding.} We demonstrate that signal--utility 
%direction reverses across environments and model backbones, and 
%that the most informative signal itself varies across settings.
%\item \textbf{Theoretical result.} We trace this reversal to two 
%coexisting state types with opposite signal--utility relationships, 
%and prove that no fixed-direction gate can be universally correct.
%\item \textbf{Adaptive gating method.} We propose DIAL, which 
%learns signal direction per environment instead of assuming it, 
%outperforming all fixed-direction baselines across six environments 
%with zero deployment overhead.
%\end{itemize}

\section{Related Work}
\vspace{-0.1in}
\label{sec:related}
Prior work on adaptive test-time compute uses diverse mechanisms, but shares a common assumption: gating signals have a fixed, setting-independent relationship with compute value. We organize the literature by how each line embeds this assumption.

\textbf{Gating on uncertainty signals.}
Signal-based methods allocate extra computation when token confidence, entropy, or disagreement exceeds a threshold, encoding the fixed-direction assumption directly in their trigger rule. In reasoning settings, this takes the form of extended chain-of-thought generation~\cite{seag2025, cats2025, corefine2026, thinkjustenough2025, s1_2025, han2025tokenbudget}. In agent settings, vote-based methods~\cite{catts2026} invoke an arbiter when candidate actions disagree, and verbalized uncertainty methods~\cite{auq2026} trigger reflection when self-reported confidence is low. All of these methods hardcode the direction from uncertainty to the value of computation; we show this direction is itself a function of (environment, backbone) and cannot be hardcoded.

\textbf{Gating on learned difficulty.}
Hidden-state probes~\cite{diffadapt2025, zhu2025llmalreadyknows} estimate task difficulty from internal representations, and RL-based methods~\cite{adaptthink2025, thinkless2025, aggarwal2025l1, paglieri2025learning} learn think/no-think policies end-to-end. Although these methods commit to neither a specific signal nor an explicit threshold, their evaluations primarily focus on homogeneous reasoning benchmarks such as GSM8K~\cite{DBLP:journals/corr/abs-2110-14168} and GPQA~\cite{DBLP:journals/corr/abs-2311-12022}, where higher difficulty is typically aligned with greater benefit from additional reasoning. The difficulty-compute relationship is thus never tested for reversal. We test it explicitly in heterogeneous agent environments, where it fails.

\textbf{Confidence calibration vs.\ utility calibration.}
Classical confidence calibration aligns a model's stated confidence with its action correctness; recent work shows that calibration quality is domain-dependent~\cite{tao2025revisiting, heo2025llmuncertainty, yadkori2024believe}. Our problem is distinct: \emph{utility calibration} aligns gating signals with the value of additional computation, not action correctness. Rational metareasoning~\cite{russell1991right, desabbata2024metareasoning} formalizes this VOC target but classically assumes $\text{VOC} \geq 0$; we show this fails when the same model both proposes and evaluates actions (\S\ref{sec:toy-model}). Beyond domain-dependent quality, the \emph{sign} of the signal-to-utility correlation reverses across (environment, backbone) pairs, so even correctness-calibrated confidence is insufficient for gating. Full discussion in Appendix~\ref{app:related}.

% ============================================================
% 3. WHY FIXED-DIRECTION GATING FAILS
% ============================================================
\vspace{-0.1in}
\section{Why Fixed-Direction Gating Fails}
\label{sec:signal-landscape}
\vspace{-0.1in}
In this section, we first formalize the adaptive gating problem (\S\ref{sec:setup}), then measure signal-utility correlations across diverse environments to test the fixed-direction assumption (\S\ref{sec:empirical-landscape}), and finally model why the assumption fails (\S\ref{sec:toy-model}).

\vspace{-0.1in}
\subsection{Problem Setup}
\label{sec:setup}
We study LLM-based agents that solve tasks through multi-step interaction with an environment. At each step $t$, the agent observes state $s_t$, selects action $a_t$ from its base policy $\pi$, and transitions to $s_{t+1}$. An environment-specific \emph{test time optimizer} $T$ (e.g., rollout evaluation, $K$-variant sampling) can be invoked at any step to potentially improve action quality. Let $\sigma(s)$ denote a scalar gating signal at state $s$ (e.g., token entropy, confidence, disagreement, or a learned difficulty score). We define the \emph{optimizer utility} as:
\begin{equation}
  U(T, s) = \mathbb{E}[R(\tau) \mid a{=}T(s)] - \mathbb{E}[R(\tau) \mid a{=}\pi(s)],
  \label{eq:utility}
\end{equation}
where $R(\tau)$ is the episode return. $U > 0$ indicates the optimizer improves the outcome; $U \leq 0$ indicates it is wasteful or harmful. The \emph{adaptive gating problem} is to learn a gate $g: \mathcal{S} \to \{0, 1\}$ that triggers $T$ only when $U > 0$.

\subsection{Signal-Utility Direction Reverses Across Environments and Backbones}
\label{sec:empirical-landscape}

\begin{wraptable}{r}{0.52\textwidth}
	\vspace{-12pt}
	\caption{Signal-utility correlations across 6 environments and 3 backbones. $\rho$: Spearman of token entropy with utility $U$; direction varies across both. Cons.: sign agreement across backbones (\cmark same, \xmark flip, ---: insufficient). $^*$: $p{<}0.05$.}
%	\vspace{-0.1in}
	\label{tab:signal-discovery}
	\centering\footnotesize
	\setlength{\tabcolsep}{3pt}
	\begin{tabular}{lcccc}
		\toprule
		Environment & Qwen3 & Phi-3.5 & Llama-3.1 & Cons. \\
		\midrule
		HotpotQA   & $-$0.327$^*$ & $+$0.184$^*$ & $-$0.346$^*$ & \xmark \\
		WebShop    & $+$0.133$^*$ & $+$0.335$^*$ & $+$0.287$^*$ & \cmark \\
		FEVER      & $-$0.119$^*$ & $-$0.156$^*$ & $+$0.428$^*$ & \xmark \\
		APPS       & $+$0.317$^*$ & $-$0.024     & $-$0.249$^*$ & \xmark \\
		TWExpress  & $-$0.290$^*$ & $+$0.000     & $+$0.000     & --- \\
		Plancraft  & $-$0.016     & $+$0.000     & $-$0.176$^*$ & --- \\
		\bottomrule
	\end{tabular}
	\vspace{-10pt}
\end{wraptable}

\paragraph{The signal-utility direction is not fixed.}
Table~\ref{tab:signal-discovery} reports Spearman correlations between token entropy and optimizer utility across 6 environments and 3 model backbones~\cite{DBLP:journals/corr/abs-2505-09388, DBLP:journals/corr/abs-2407-21783, DBLP:journals/corr/abs-2404-14219}. The correlation reverses along two axes. First, \emph{the same signal can reverse across backbones with the task held fixed}: among environments with at least two significant nonzero correlations, sign disagreement across backbones is common. The most extreme case is FEVER, which flips from $\rho{=}{-}0.156$ on Phi-3.5 to $\rho{=}{+}0.428$ on Llama-3.1, a sign change driven entirely by the choice of LLM. The signal's semantics are therefore not a property of the environment alone, but of how the agent's prior knowledge interacts with it. Second, \emph{within a single backbone, the direction varies across environments}: on Qwen3, entropy is negatively correlated with utility in information-seeking environments (FEVER, HotpotQA, TWExpress) but positively in code generation (APPS) and web navigation (WebShop), with near-zero correlation in weak-signal environments (Plancraft).

\textbf{Signal identity also varies across environments.}
Beyond direction, the most informative signal itself differs across environments: \texttt{step\_count} dominates in most settings, but WebShop relies on \texttt{num\_available\_actions} instead (Appendix~\ref{app:feature-selection}), and no single signal is reliably informative across all environments (Appendix~\ref{app:auc}).

\textbf{Wrong direction is harmful, not just uninformative.}
This direction mismatch is not benign: a gate that assumes the wrong direction systematically falls \emph{below} the no-trigger baseline. On strong-signal environments, reversing the gate's direction collapses success rate by 23–37 points (Table~\ref{tab:wrong-direction}), making it strictly worse than doing nothing in mismatched (environment, backbone) pairs.

\textbf{The reversal is not an artifact of entropy scales.}
A natural concern is that cross-backbone reversal could reflect incomparable raw entropy scales across tokenizers. We rule this out: per-(environment, backbone) quantile normalization preserves both Spearman rank ordering (trivially) and Pearson sign on the strong-correlation cells (Appendix~\ref{app:reversal-norm}). The reversal is also robust to monotone reward transformations $U \!\to\! \alpha U$ ($\alpha > 0$) and to entropy transformations $\sigma \!\to\! \sigma/T$, $\sigma \!\to\! \sigma^\alpha$, and $\sigma \!\to\! \log\sigma$, ruling out reward parameterization and entropy-pipeline calibration as alternative explanations (Appendix~\ref{app:reversal-alternatives}). Together, these findings suggest that uncertainty reflects latent state structure that varies across (environment, backbone) pairs, which we formalize next.

\vspace{-0.1in}
\subsection{Why Does Direction Reverse? A Two-Source Model}
\label{sec:toy-model}
\vspace{-0.1in}
\textbf{Model.}
We model direction reversal as arising from two coexisting regimes within each environment, not from two disjoint classes of tasks. The following linear form is a regime-level abstraction: it specifies the sign of the signal-utility relationship within each regime, not a mechanistic claim that all states within a regime share a single underlying cause. Operationally, Type~I denotes regimes with a negative within-regime signal-utility slope, regardless of whether the cause is missing evidence, weak grounding, invalid intermediate states, or rollout-induced compounding errors. At each decision step, a state can be of either type:
\begin{itemize}[leftmargin=*,nosep]
\item \textbf{Type~I (intervention-unsuitable):} Utility is \emph{negatively} related to entropy: $U_I(s) \sim -\alpha H(s) + \varepsilon_I$, with $\alpha > 0$ and zero-mean noise $\varepsilon_I$. These are states where the current context does not provide a reliable basis for the rollout optimizer to improve the action; missing evidence in evidence-seeking tasks is one common cause, alongside misleading partial observations, weak grounding, or compounding rollout errors.
\item \textbf{Type~D (decision-difficult):} Utility is \emph{positively} related to entropy: $U_D(s) \sim +\beta H(s) + \varepsilon_D$, with $\beta > 0$ and zero-mean noise $\varepsilon_D$. These are states where uncertainty reflects multiple viable alternatives, so rollouts can productively sample, compare, or verify candidate actions.
\end{itemize}
Here $\alpha$ and $\beta$ are the within-type slopes: they capture how strongly entropy maps to utility within each state type. A real environment $\mathcal{E}$ contains a mixture of both types; let $p_I(\mathcal{E})$ denote the fraction of Type~I states. The marginal entropy-utility correlation across all states is a weighted blend of the two within-type slopes:
\begin{equation}
  \rho(\mathcal{E}) \;\approx\; \beta - (\alpha + \beta)\,p_I(\mathcal{E}).
  \label{eq:direction}
\end{equation}
At the two extremes, $\rho \!=\! +\beta$ when the environment is pure Type~D ($p_I\!=\!0$) and $\rho \!=\! -\alpha$ when it is pure Type~I ($p_I\!=\!1$). In between, $\rho$ crosses zero at $p_I^* = \beta/(\alpha+\beta)$: environments above this threshold show negative $\rho$ (Type~I dominated), those below show positive $\rho$ (Type~D dominated), and those near it show negligible signal as the two effects cancel. This expression captures the sign-level dependence under a simplified linear model and is not intended as an estimator of Spearman $\rho$. 
%The prediction is consistent with the observed Qwen3 pattern in Table~\ref{tab:signal-discovery}: search-based QA (FEVER, HotpotQA, TWExpress) shows negative $\rho$, code generation (APPS) shows positive $\rho$, and weak-signal environments (Plancraft) cluster near zero.

\textbf{Implication: $\sigma$ is not sufficient for optimizer utility.}
Two states with identical $\sigma$ but opposite latent types have opposite-sign utility. An uncertainty-only gate $g(\sigma(s))$ that lacks features distinguishing latent state types therefore cannot guarantee non-negative value across heterogeneous settings (Proposition~\ref{prop:necessity}, proof in Appendix~\ref{app:proof-necessity}). Recovering sufficiency requires augmenting $\sigma$ with features that distinguish the two state types, directly motivating the multi-feature design of \S\ref{sec:method}.

\textbf{Backbone-level reversal is a prediction, not an anomaly.}
The fraction $p_I(\mathcal{E})$ defined above depends on more than just the environment. A state is intervention-unsuitable only relative to what the agent already encodes: the same observation can be Type~I for a model that lacks the relevant knowledge or grounding and Type~D for one that has it. The effective mixture is therefore indexed by the (environment, backbone) pair, not the environment alone. This explains the backbone-level sign flips of Table~\ref{tab:signal-discovery}, most extreme on FEVER between Phi-3.5 and Llama-3.1: backbone changes shift $p_I$, and when they cross $p_I^*$, $\rho$ flips (Appendix~\ref{app:multi-backbone}).

\textbf{Theoretical grounding.}
This direction reversal is a special case of Simpson's paradox~\cite{simpson1951interpretation, DBLP:books/acm/22/Pearl22n}: aggregating subpopulations with opposing within-group trends reverses the aggregate correlation. Our decomposition adapts the epistemic/aleatoric distinction~\cite{der2009aleatory} to adaptive compute: Type~I states often exhibit epistemic-like uncertainty (e.g., information deficit or weak grounding), while Type~D states exhibit aleatoric-like diversity (multiple viable paths). Rather than a post-hoc explanation, this decomposition makes testable predictions: in \S\ref{sec:theory-verification} we derive three (temporal dynamics, cross-environment consistency, and signal identity alignment) and verify each on held-out data.

% ============================================================
% 4. METHOD: DIAL
% ============================================================
\begin{figure*}[t]
	\centering
	\includegraphics[width=\textwidth]{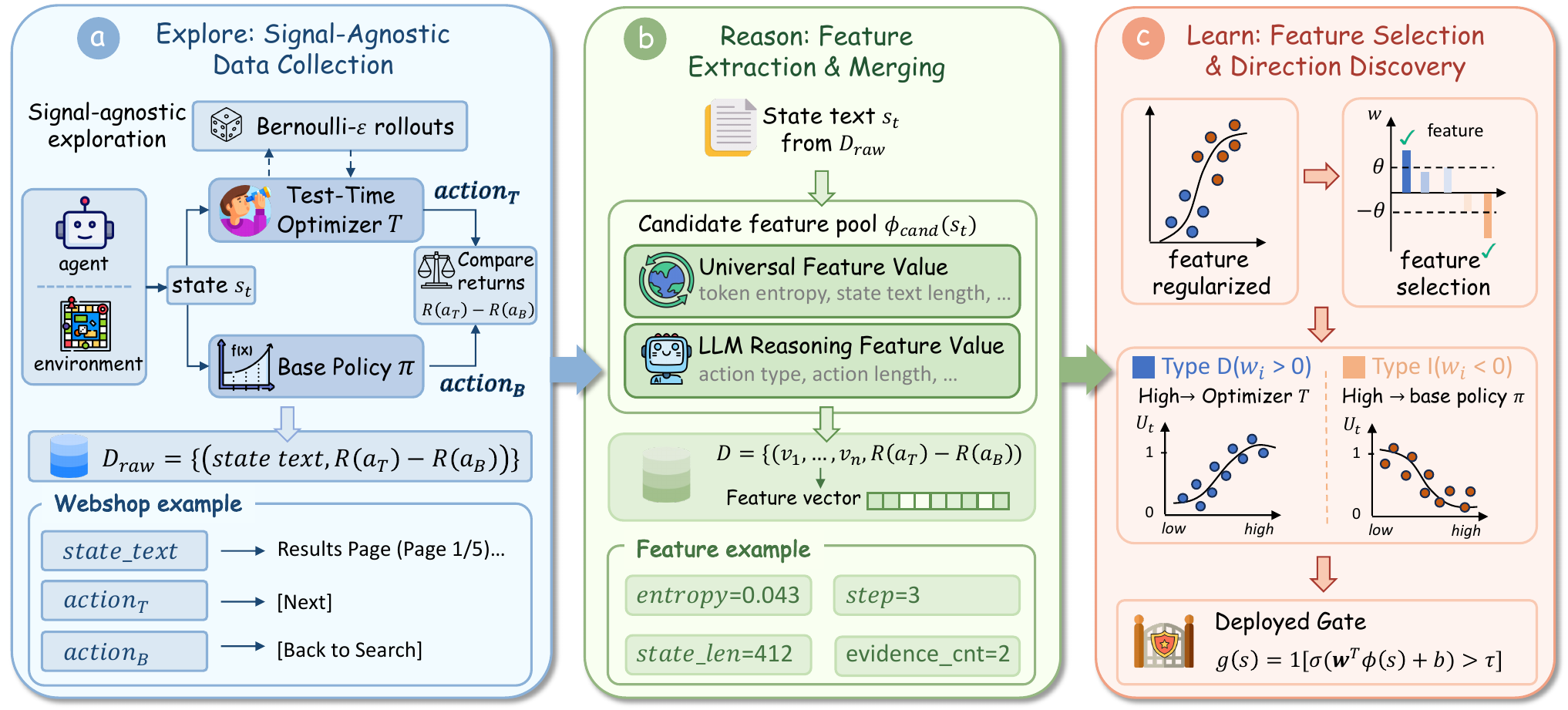} 
    \vspace{-0.2in}
	\caption{\textbf{The DIAL pipeline.}
		\textbf{(a) Explore}: paired counterfactual rollouts via Bernoulli-$\varepsilon$ yield raw data $\mathcal{D}_{\text{raw}}{=}\{(\text{state text}, R(a_T){-}R(a_B))\}$.
		\textbf{(b) Reason}: from each state text, a candidate feature pool $\phi_{\text{cand}}(s)$ combines universal features with LLM-proposed task-specific features, yielding $\mathcal{D}{=}\{(\phi(s), U)\}$.
		\textbf{(c) Learn}: fit an $\ell_1$-regularized logistic gate $g(s)$; signed weights jointly select features and recover per-environment direction (Type~D: $w_i{>}0$; Type~I: $w_i{<}0$).}
	\label{fig:method_overview}
\end{figure*}
\vspace{-0.1in}
\section{Direction-Informed Adaptive Learning}
\label{sec:method}
\vspace{-0.1in}
DIAL reframes gating from difficulty estimation to intervention-utility estimation: it trains on counterfactual optimizer utility rather than difficulty or correctness proxies, and uses signal-agnostic exploration so that the learned gate can discover when each feature indicates positive or negative rollout value instead of inheriting a fixed direction from the data collection rule.
The analysis in \S\ref{sec:signal-landscape} implies three constraints for any cross-setting adaptive gate. It must (i) \emph{discover} the signal-utility direction per (environment, backbone) pair rather than hardcoding it; (ii) draw from a \emph{multi-signal feature pool} that approximates the latent state type $\tau(s)$, since no single signal is reliable across settings; and (iii) collect training data \emph{without conditioning on} the very signal whose direction is unknown, otherwise the data is biased toward the prior assumption. DIAL satisfies these three constraints in three stages (Fig.~\ref{fig:method_overview}): randomized exploration (\S\ref{sec:explore}), feature construction (\S\ref{sec:reason}), and sparse gate fitting (\S\ref{sec:learn}).

\vspace{-0.1in}
\subsection{Explore: Signal-Agnostic Data Collection}
\label{sec:explore}
\vspace{-0.05in}
The exploration stage collects the $(s, U)$ pairs used to fit the gate, under one central constraint: the collection policy must remain independent of any signal whose semantics we are trying to learn. A gate trained on uncertainty-triggered states would systematically under-sample the states that determine $\rho$'s sign, baking the prior assumption into the training data instead of testing it.

DIAL therefore uses the simplest policy consistent with this constraint. At each step~$t$ in episode~$i$, a uniform random coin flip with probability $\varepsilon_{\text{explore}}$ determines whether to trigger the optimizer~$T$ ($z_t^{(i)}{=}1$) or use the base action ($z_t^{(i)}{=}0$). For each step with $z_t^{(i)}{=}1$, we estimate utility by paired counterfactual rollouts from the same state snapshot: we fork the environment, execute both the optimizer's chosen action and the base action via truncated rollouts of horizon $H$, and record which yields a higher return. We set $U_t^{(i)} = 1$ if the optimizer's action wins, and $0$ otherwise; we binarize the continuous utility of Eq.~\eqref{eq:utility} for logistic-gate fitting, which preserves the sign that drives gating decisions. The Bernoulli trigger eliminates selection bias, and the paired design eliminates baseline bias; the remaining approximation is the truncation horizon $H$. Aggregating across $N_{\text{explore}}$ episodes yields the dataset:
\begin{equation}
  \mathcal{D} = \bigl\{
    \bigl(\phi(s_t^{(i)}),\; U_t^{(i)}\bigr)
    : z_t^{(i)} = 1
  \bigr\}_{t, i},
  \label{eq:dataset}
\end{equation}
where $\phi(s)$ is the feature vector. Per-environment estimation details are in Appendix~\ref{app:environment}.

\vspace{-0.15in}
\subsection{Reason: Feature Construction}
\label{sec:reason}
\vspace{-0.05in}
The next stage builds the candidate feature pool $\phi_{\text{cand}}(s)$. This pool must be rich enough to capture each environment's informative signals, but not so environment-specific that deployment requires manual engineering per task.

DIAL's candidate pool combines two layers. The first is a small set of universal features defined identically across environments: \texttt{step\_count}, \texttt{token\_entropy}, \texttt{evidence\_count}, \texttt{num\_available\_actions}, and \texttt{is\_finish} (definitions in Table~\ref{tab:universal-features}). All universal features are computed from the state text via either direct parsing (e.g., \texttt{evidence\_count}) or the agent's existing LLM forward pass (e.g., \texttt{token\_entropy}, \texttt{is\_finish}); no additional data source is required. Signals an environment does not naturally expose (e.g., \texttt{evidence\_count} outside QA-style tasks) default to zero, requiring no per-environment engineering. The second is an LLM-generated layer: an LLM receives a summary of $\mathcal{D}$ and proposes task-specific feature hypotheses $\phi_{\text{LLM}}(s)$, forming $\phi_{\text{cand}}(s) = \phi_{\text{univ}}(s) \cup \phi_{\text{LLM}}(s)$. The $\ell_1$ penalty in the next stage automatically discards uninformative proposals, so the LLM layer can propose freely without degrading gate quality. We show in \S\ref{sec:ablation} that the universal pool alone already outperforms all fixed-direction baselines, and the LLM layer adds further gains in environments where universal features are inadequate. We provide the prompt template and per-environment ablation in Appendix~\ref{app:llm-feature-layer} and Appendix~\ref{app:llm-features}.

\vspace{-0.15in}
\subsection{Learn: Sparse Linear Gate}
\label{sec:learn}
\vspace{-0.05in}

Given $\mathcal{D}$ and $\phi_{\text{cand}}$, DIAL fits an $\ell_1$-regularized logistic regression with the binary cross-entropy loss, where the regularization strength $\lambda$ controls sparsity over $\phi_{\text{cand}}$ so that uninformative features drop out. Let $(\mathbf{w}^*, b^*)$ denote the fitted weights and bias. The deployed gate is
\begin{equation}
  g(s) = \mathbf{1}\!\left[\mathrm{sigmoid}(\mathbf{w}^{*\top} \phi(s) + b^*) > \tau\right],
  \label{eq:gate}
\end{equation}
with threshold $\tau$ set by cross-validation on $\mathcal{D}$. At deployment, the gate requires only a feature evaluation and a single sigmoid: no extra LLM calls, no $K$-way voting, no calibration query. This zero per-step LLM overhead after the one-time exploration phase contrasts with prior gating methods that incur per-step costs (e.g., CATTS triggers $K{=}5$ forward passes per step, AUQ issues a confidence query every step).

\textbf{Signed weights as a diagnostic.}
The $\ell_1$ penalty simultaneously performs feature selection and direction discovery. The sign of each fitted weight provides an interpretable per-environment estimate of how the feature relates to optimizer utility: $w_i > 0$ indicates that feature~$i$ behaves as a decision-difficulty proxy in this setting, where larger values increase expected rollout value; $w_i < 0$ indicates an intervention-unsuitability proxy, where larger values mark states in which rollouts are less likely to help; and $w_i \approx 0$ indicates that the feature is uninformative. The complete pipeline, including online adaptation for drifting environments and detailed cost accounting, is in Appendix~\ref{app:dial-details}.

% ============================================================
% 5. EXPERIMENTS
% ============================================================
\vspace{-0.1in}
\section{Experiments}
\label{sec:experiments}
\vspace{-0.1in}
\subsection{Experimental Setup}
\label{sec:exp-setup}

We evaluate DIAL along three axes: does environment-aware gating improve the success-cost
frontier over fixed-direction baselines
(\S\ref{sec:main-results}); which components drive the gains
(\S\ref{sec:ablation}); and do the Two-Source Model's predictions
hold on held-out data (\S\ref{sec:theory-verification}).

\textbf{Environments.} We conduct evaluations across six environments: 
HotpotQA~\cite{yang2018hotpotqa},
APPS~\cite{hendrycks2021apps},
WebShop~\cite{yao2022webshop},
FEVER~\cite{thorne2018fever},
TWExpress~\cite{jansen2023twexpress}, and
Plancraft~\cite{dagan2024plancraft}. 
% , spanning QA, code generation,
% web navigation, fact verification, text games, and crafting planning.
These cover the full range of the Two-Source
Model, from intervention-unsuitable-leaning (FEVER) through mixed (WebShop) to decision-difficulty-leaning (APPS) settings, including extreme rollout properties
(TWExpress: rollout-safe; Plancraft: rollout-harmful).
Full environment specifications are in
Appendix~\ref{app:environment}.

\vspace{-0.025in}
\textbf{Baselines.} All methods share the same agent and
optimizer $T$; only the gating decision differs.
Six fixed-direction or fixed-budget baselines instantiate a
predetermined compute-allocation rule (high uncertainty, low
confidence, disagreement, or a fixed token budget triggering
additional computation):
CaTS~\cite{cats2025}, SEAG~\cite{seag2025},
CoRefine~\cite{corefine2026}, CATTS~\cite{catts2026},
AUQ~\cite{auq2026}, and s1\_budget~\cite{s1_2025}. Full specifications are in Appendix~\ref{app:method-comparison}. We additionally include two bounds for reference (Appendix~\ref{app:bounds-comparison}): \emph{base\_only}, which never triggers
the optimizer, and \emph{always\_trigger}.

\begin{table*}[t]
\centering
\caption{Full results across 6 environments and 3 backbones. SR (\%) and Cost ($\times$base, total deployment tokens normalized by base\_only).
\colorbox[HTML]{FBE7EB}{\strut Best} and \colorbox[HTML]{E5EDF7}{\strut second-best} mean SR per environment.
$\dagger$: methods requiring calibration data.
Cost includes all LLM calls: base proposer, gate overhead (CATTS $K{=}5$ voting, AUQ confidence query), and rollout calls.}
\vspace{-0.1in}
\label{tab:full-results}
\renewcommand{\arraystretch}{0.9}
\resizebox{\textwidth}{!}{%
\begin{tabular}{l *{6}{rc} }
\toprule
& \multicolumn{2}{c}{\textbf{HotpotQA}} & \multicolumn{2}{c}{\textbf{WebShop}} & \multicolumn{2}{c}{\textbf{FEVER}}
& \multicolumn{2}{c}{\textbf{TWExpress}} & \multicolumn{2}{c}{\textbf{Plancraft}} & \multicolumn{2}{c}{\textbf{APPS}} \\
\cmidrule(lr){2-3}\cmidrule(lr){4-5}\cmidrule(lr){6-7}
\cmidrule(lr){8-9}\cmidrule(lr){10-11}\cmidrule(lr){12-13}
\textbf{Method} & SR$\uparrow$ & Cost$\downarrow$ & SR$\uparrow$ & Cost$\downarrow$ & SR$\uparrow$ & Cost$\downarrow$ & SR$\uparrow$ & Cost$\downarrow$ & SR$\uparrow$ & Cost$\downarrow$ & SR$\uparrow$ & Cost$\downarrow$ \\
\midrule
\multicolumn{13}{c}{\textit{Qwen3-4B-Instruct}} \\
\midrule
CaTS$^\dagger$ & $93.2{\scriptscriptstyle\,\pm\,1.4}$ & 10.54 & $30.5{\scriptscriptstyle\,\pm\,2.2}$ & 3.45 & \cellcolor[HTML]{E5EDF7}$48.2{\scriptscriptstyle\,\pm\,3.6}$ & 25.48 & $96.7{\scriptscriptstyle\,\pm\,0.8}$ & 3.37 & $20.3{\scriptscriptstyle\,\pm\,2.8}$ & 4.14 & $66.2{\scriptscriptstyle\,\pm\,0.0}$ & 2.61 \\
SEAG$^\dagger$ & $67.5{\scriptscriptstyle\,\pm\,4.8}$ & 9.34 & $28.0{\scriptscriptstyle\,\pm\,4.9}$ & 2.84 & $47.3{\scriptscriptstyle\,\pm\,3.2}$ & 17.19 & $97.3{\scriptscriptstyle\,\pm\,0.3}$ & 2.60 & $22.8{\scriptscriptstyle\,\pm\,0.8}$ & 3.52 & $66.0{\scriptscriptstyle\,\pm\,0.0}$ & 2.75 \\
CoRefine$^\dagger$ & $68.2{\scriptscriptstyle\,\pm\,4.9}$ & 9.28 & $27.5{\scriptscriptstyle\,\pm\,3.8}$ & 2.78 & $47.8{\scriptscriptstyle\,\pm\,2.0}$ & 17.19 & \cellcolor[HTML]{E5EDF7}$97.5{\scriptscriptstyle\,\pm\,0.5}$ & 2.57 & $20.8{\scriptscriptstyle\,\pm\,3.0}$ & 3.44 & $67.5{\scriptscriptstyle\,\pm\,0.0}$ & 2.70 \\
CATTS & $68.3{\scriptscriptstyle\,\pm\,4.6}$ & 10.53 & $16.0{\scriptscriptstyle\,\pm\,0.5}$ & 5.55 & $32.2{\scriptscriptstyle\,\pm\,0.3}$ & 16.46 & $97.5{\scriptscriptstyle\,\pm\,0.5}$ & 2.83 & \cellcolor[HTML]{E5EDF7}$23.0{\scriptscriptstyle\,\pm\,2.0}$ & 6.53 & $60.8{\scriptscriptstyle\,\pm\,0.0}$ & 6.00 \\
AUQ & \cellcolor[HTML]{E5EDF7}$95.0{\scriptscriptstyle\,\pm\,0.5}$ & 10.29 & \cellcolor[HTML]{E5EDF7}$35.7{\scriptscriptstyle\,\pm\,8.3}$ & 5.83 & $38.7{\scriptscriptstyle\,\pm\,4.0}$ & 18.04 & $95.5{\scriptscriptstyle\,\pm\,0.0}$ & 1.95 & $22.2{\scriptscriptstyle\,\pm\,1.6}$ & 6.64 & $64.7{\scriptscriptstyle\,\pm\,1.0}$ & 3.28 \\
s1\_budget & $94.0{\scriptscriptstyle\,\pm\,0.5}$ & 8.27 & $17.8{\scriptscriptstyle\,\pm\,2.6}$ & 2.91 & $44.2{\scriptscriptstyle\,\pm\,2.0}$ & 19.21 & $95.0{\scriptscriptstyle\,\pm\,0.5}$ & 2.02 & $16.3{\scriptscriptstyle\,\pm\,2.4}$ & 3.02 & \cellcolor[HTML]{E5EDF7}$69.0{\scriptscriptstyle\,\pm\,1.3}$ & 2.65 \\
\textbf{DIAL} & \cellcolor[HTML]{FBE7EB}$95.2{\scriptscriptstyle\,\pm\,0.3}$ & 8.02 & \cellcolor[HTML]{FBE7EB}$43.8{\scriptscriptstyle\,\pm\,4.9}$ & 2.50 & \cellcolor[HTML]{FBE7EB}$49.8{\scriptscriptstyle\,\pm\,2.5}$ & 16.51 & \cellcolor[HTML]{FBE7EB}$99.0{\scriptscriptstyle\,\pm\,0.3}$ & 1.81 & \cellcolor[HTML]{FBE7EB}$23.3{\scriptscriptstyle\,\pm\,1.5}$ & 3.63 & \cellcolor[HTML]{FBE7EB}$73.0{\scriptscriptstyle\,\pm\,1.0}$ & 2.61 \\
\midrule
\multicolumn{13}{c}{\textit{Phi-3.5-mini-Instruct}} \\
\midrule
CaTS$^\dagger$ & $68.3{\scriptscriptstyle\,\pm\,2.8}$ & 5.86 & $41.7{\scriptscriptstyle\,\pm\,3.8}$ & 2.52 & $19.8{\scriptscriptstyle\,\pm\,2.0}$ & 14.51 & $68.2{\scriptscriptstyle\,\pm\,4.2}$ & 1.00 & $13.5{\scriptscriptstyle\,\pm\,0.5}$ & 1.44 & $30.5{\scriptscriptstyle\,\pm\,1.8}$ & 3.55 \\
SEAG$^\dagger$ & $88.3{\scriptscriptstyle\,\pm\,3.1}$ & 7.82 & $36.5{\scriptscriptstyle\,\pm\,2.8}$ & 2.43 & $13.5{\scriptscriptstyle\,\pm\,3.0}$ & 5.61 & $92.5{\scriptscriptstyle\,\pm\,1.5}$ & 2.58 & $14.7{\scriptscriptstyle\,\pm\,0.8}$ & 2.29 & $27.8{\scriptscriptstyle\,\pm\,0.8}$ & 2.59 \\
CoRefine$^\dagger$ & $87.7{\scriptscriptstyle\,\pm\,2.8}$ & 7.83 & $35.8{\scriptscriptstyle\,\pm\,3.3}$ & 2.44 & $13.7{\scriptscriptstyle\,\pm\,3.3}$ & 5.73 & $92.7{\scriptscriptstyle\,\pm\,1.5}$ & 2.56 & \cellcolor[HTML]{E5EDF7}$15.2{\scriptscriptstyle\,\pm\,1.3}$ & 2.24 & $28.5{\scriptscriptstyle\,\pm\,0.3}$ & 2.60 \\
CATTS & $42.0{\scriptscriptstyle\,\pm\,1.5}$ & 7.98 & $28.7{\scriptscriptstyle\,\pm\,1.9}$ & 4.86 & $12.7{\scriptscriptstyle\,\pm\,1.3}$ & 7.65 & $86.7{\scriptscriptstyle\,\pm\,2.3}$ & 4.00 & $13.5{\scriptscriptstyle\,\pm\,0.0}$ & 5.96 & $27.7{\scriptscriptstyle\,\pm\,0.0}$ & 6.10 \\
AUQ & $31.2{\scriptscriptstyle\,\pm\,1.4}$ & 5.47 & $37.0{\scriptscriptstyle\,\pm\,2.0}$ & 3.11 & $8.5{\scriptscriptstyle\,\pm\,2.3}$ & 3.18 & $92.5{\scriptscriptstyle\,\pm\,1.5}$ & 1.82 & $12.7{\scriptscriptstyle\,\pm\,0.3}$ & 2.24 & $27.2{\scriptscriptstyle\,\pm\,0.3}$ & 2.03 \\
s1\_budget & \cellcolor[HTML]{E5EDF7}$91.2{\scriptscriptstyle\,\pm\,2.9}$ & 5.26 & \cellcolor[HTML]{E5EDF7}$43.0{\scriptscriptstyle\,\pm\,4.5}$ & 3.89 & \cellcolor[HTML]{E5EDF7}$20.0{\scriptscriptstyle\,\pm\,5.4}$ & 8.45 & \cellcolor[HTML]{E5EDF7}$93.7{\scriptscriptstyle\,\pm\,0.8}$ & 2.05 & $13.7{\scriptscriptstyle\,\pm\,0.8}$ & 1.73 & \cellcolor[HTML]{E5EDF7}$34.5{\scriptscriptstyle\,\pm\,0.3}$ & 2.93 \\
\textbf{DIAL} & \cellcolor[HTML]{FBE7EB}$92.3{\scriptscriptstyle\,\pm\,2.8}$ & 5.38 & \cellcolor[HTML]{FBE7EB}$57.3{\scriptscriptstyle\,\pm\,1.3}$ & 2.56 & \cellcolor[HTML]{FBE7EB}$20.5{\scriptscriptstyle\,\pm\,3.0}$ & 7.44 & \cellcolor[HTML]{FBE7EB}$96.7{\scriptscriptstyle\,\pm\,0.8}$ & 1.83 & \cellcolor[HTML]{FBE7EB}$16.8{\scriptscriptstyle\,\pm\,0.8}$ & 1.33 & \cellcolor[HTML]{FBE7EB}$36.8{\scriptscriptstyle\,\pm\,2.0}$ & 2.87 \\
\midrule
\multicolumn{13}{c}{\textit{Llama-3.1-8B-Instruct}} \\
\midrule
CaTS$^\dagger$ & $94.2{\scriptscriptstyle\,\pm\,1.3}$ & 7.21 & $36.0{\scriptscriptstyle\,\pm\,6.3}$ & 2.51 & \cellcolor[HTML]{FBE7EB}$56.3{\scriptscriptstyle\,\pm\,2.3}$ & 14.82 & $35.8{\scriptscriptstyle\,\pm\,5.4}$ & 3.00 & $22.8{\scriptscriptstyle\,\pm\,0.3}$ & 3.42 & $55.2{\scriptscriptstyle\,\pm\,1.3}$ & 4.81 \\
SEAG$^\dagger$ & $94.0{\scriptscriptstyle\,\pm\,2.0}$ & 9.82 & $27.5{\scriptscriptstyle\,\pm\,3.9}$ & 1.86 & $53.5{\scriptscriptstyle\,\pm\,2.9}$ & 12.57 & \cellcolor[HTML]{E5EDF7}$76.5{\scriptscriptstyle\,\pm\,0.9}$ & 3.85 & $21.8{\scriptscriptstyle\,\pm\,3.5}$ & 2.70 & \cellcolor[HTML]{E5EDF7}$57.2{\scriptscriptstyle\,\pm\,1.6}$ & 4.41 \\
CoRefine$^\dagger$ & $94.2{\scriptscriptstyle\,\pm\,1.5}$ & 9.42 & $27.5{\scriptscriptstyle\,\pm\,3.8}$ & 1.82 & $53.5{\scriptscriptstyle\,\pm\,3.3}$ & 12.64 & $76.0{\scriptscriptstyle\,\pm\,0.0}$ & 3.56 & $21.5{\scriptscriptstyle\,\pm\,2.2}$ & 2.69 & $57.2{\scriptscriptstyle\,\pm\,1.3}$ & 4.37 \\
CATTS & $94.0{\scriptscriptstyle\,\pm\,2.0}$ & 11.11 & $11.3{\scriptscriptstyle\,\pm\,3.8}$ & 4.74 & $14.2{\scriptscriptstyle\,\pm\,2.8}$ & 17.00 & $49.3{\scriptscriptstyle\,\pm\,0.0}$ & 4.11 & \cellcolor[HTML]{FBE7EB}$27.0{\scriptscriptstyle\,\pm\,2.9}$ & 7.31 & $52.3{\scriptscriptstyle\,\pm\,1.0}$ & 6.70 \\
AUQ & \cellcolor[HTML]{E5EDF7}$95.0{\scriptscriptstyle\,\pm\,1.5}$ & 10.54 & \cellcolor[HTML]{E5EDF7}$37.8{\scriptscriptstyle\,\pm\,4.3}$ & 4.54 & $52.8{\scriptscriptstyle\,\pm\,2.3}$ & 25.08 & $63.2{\scriptscriptstyle\,\pm\,0.9}$ & 3.58 & $27.0{\scriptscriptstyle\,\pm\,1.5}$ & 8.21 & $46.0{\scriptscriptstyle\,\pm\,0.5}$ & 6.11 \\
s1\_budget & $94.5{\scriptscriptstyle\,\pm\,0.5}$ & 8.08 & $11.2{\scriptscriptstyle\,\pm\,0.5}$ & 1.87 & $13.8{\scriptscriptstyle\,\pm\,1.0}$ & 8.56 & $76.1{\scriptscriptstyle\,\pm\,3.0}$ & 3.05 & $26.3{\scriptscriptstyle\,\pm\,0.8}$ & 2.53 & $52.5{\scriptscriptstyle\,\pm\,0.8}$ & 2.33 \\
\textbf{DIAL} & \cellcolor[HTML]{FBE7EB}$95.5{\scriptscriptstyle\,\pm\,1.7}$ & 7.07 & \cellcolor[HTML]{FBE7EB}$41.7{\scriptscriptstyle\,\pm\,4.5}$ & 1.68 & \cellcolor[HTML]{E5EDF7}$54.7{\scriptscriptstyle\,\pm\,3.2}$ & 18.69 & \cellcolor[HTML]{FBE7EB}$94.8{\scriptscriptstyle\,\pm\,5.1}$ & 2.55 & \cellcolor[HTML]{E5EDF7}$27.0{\scriptscriptstyle\,\pm\,3.4}$ & 2.49 & \cellcolor[HTML]{FBE7EB}$59.7{\scriptscriptstyle\,\pm\,1.8}$ & 4.29 \\
\bottomrule
\end{tabular}%
}
\end{table*}

\vspace{-0.025in}
% \textbf{Metrics.} We use SR (success rate, $\uparrow$) and Cost
% ($\times$base, $\downarrow$): total deployment tokens normalized by
% base\_only tokens per episode. This metric counts \emph{all} LLM
% calls, including base proposer, gate overhead, and rollout calls, ensuring that per-step
% overhead is not hidden. Calibration or exploration overhead is
% excluded for all methods.
% A method Pareto-dominates another if $\text{SR} \geq$ and
% $\text{Cost} \leq$ with at least one strict inequality.
% Token cost constants are reported in
% Appendix~\ref{app:cost}.
\textbf{Metrics.} We use SR (success rate, $\uparrow$) and Cost
($\times$base, $\downarrow$), total deployment tokens normalized by
base\_only per episode and including all LLM calls (base proposer, gate
overhead, rollout); calibration and exploration overhead are excluded
for all methods. A method Pareto-dominates another if
$\text{SR}\geq$ and $\text{Cost}\leq$ with at least one strict inequality.
Token cost constants are in Appendix~\ref{app:cost}.
\vspace{-0.15in}
\subsection{Main Results}
\label{sec:main-results}
\vspace{-0.1in}

\textbf{DIAL improves the success-cost frontier over fixed-direction baselines.}
Table~\ref{tab:full-results} reports SR and Cost ($\times$base) across all 6 environments and 3 backbones; the Qwen3-4B Pareto frontier is visualized in Appendix~\ref{app:multi-backbone}. DIAL
achieves the highest SR in 16 of 18 (environment, backbone) cells. Three insights follow.
First, \emph{wrong direction is catastrophic, not merely costly}.
On Qwen3-4B FEVER, an intervention-unsuitable-leaning environment, CATTS pays 16.46$\times$ base for $K{=}5$ voting yet reaches only 32.2\% SR, below the 37.0\% no-trigger baseline: fixed-direction gating in the wrong direction
does not merely waste compute, it actively hurts.
Second, \emph{learning the right direction yields a better SR-cost trade-off, not just a retuned threshold}. On Qwen3-4B HotpotQA, DIAL trades 1.8 SR points for 1.3$\times$ less cost than always\_trigger (95.2 vs.\ 97.0\% at 8.02$\times$ vs.\ 10.63$\times$); on WebShop, DIAL matches always\_trigger at less than half the cost. The gap is too large
to close by retuning thresholds without changing direction.
Third, \emph{the win generalizes across backbones}. The same DIAL pipeline takes 6/6 cells on both Qwen3-4B and Phi-3.5-mini and 4/6 on Llama-3.1-8B, even though entropy--utility correlations flip sign across model families (\S\ref{sec:empirical-landscape}). Fixed-direction baselines instead break per-backbone: AUQ on HotpotQA falls from 95.0\% (Qwen3-4B) to 31.2\% (Phi-3.5).

% FEVER remains DIAL's weakest environment: while DIAL achieves the
% highest SR (49.8\%), its cost (16.51$\times$base) is also the
% highest, a limitation we analyze in \S\ref{sec:discussion}.
\vspace{-0.15in}
\subsection{Ablation Studies}
\label{sec:ablation}
\vspace{-0.1in}
The main results show that DIAL works; we now ask
\emph{why}. We ablate direction along two axes: its sign
(direction reversal) and the gate's capacity to exploit it
(gate complexity). We then characterize how DIAL's trigger
rate adapts across rollout-headroom regimes and explain why fewer
rollouts can improve accuracy. LLM feature and
regularizer ablations appear in Appendix~\ref{app:llm-features}
and Appendix~\ref{app:regularizer-ablation}.

\begin{wraptable}{r}{0.52\textwidth}
\vspace{-2pt}
\caption{Effect of reversing DIAL's learned direction.
$\Delta$SR: reversed $-$ original. $|\rho^*|$ is the
absolute Spearman correlation between the dominant signal
and optimizer utility in each environment. Degradation
scales with $|\rho^*|$.}
\vspace{-.05in}
\label{tab:wrong-direction}
\centering
\resizebox{0.5\textwidth}{!}{%
\begin{tabular}{lcccc}
\toprule
Environment & DIAL SR & Reversed SR & $\Delta$ SR & $|\rho^*|$ \\
\midrule
FEVER          & 49.8\% & 13.0\% & $-$36.8 & 0.619 \\
HotpotQA       & 95.2\% & 58.2\% & $-$37.0 & 0.494 \\
WebShop        & 43.8\% & 20.6\% & $-$23.2 & 0.444 \\
APPS           & 73.0\% & 70.5\% & $-$2.5  & 0.339 \\
TWExpress      & 99.0\% & 96.2\% & $-$2.8  & 0.290 \\
Plancraft      & 23.3\% & 20.5\% & $-$2.8  & 0.016 \\
\bottomrule
\end{tabular}%
}
\vspace{-10pt}
\end{wraptable}
\vspace{-0.025in}
\textbf{Direction is a first-order determinant of gate quality.}
To isolate the effect of direction, we reverse the sign of all learned weights in DIAL's fitted gate while keeping everything
else unchanged, constructing an adversarially wrong-direction gate that preserves ranking structure but inverts trigger
semantics. Table~\ref{tab:wrong-direction} shows that on strong-signal environments with $|\rho^*| > 0.4$,
reversal collapses SR by 23--37\%; on weak-signal Plancraft ($|\rho^*|{=}0.016$) the effect is negligible. The damage scales with signal strength: wrong direction is not a calibration error but a structural failure. This instantiates the
wrong-direction harm result of Appendix~\ref{app:wrong-direction}: a sharper gate fires precisely on the states where
computation has negative value.

\begin{wrapfigure}{r}{0.45\textwidth}
\vspace{-15pt}
\centering
\includegraphics[width=0.43\textwidth]{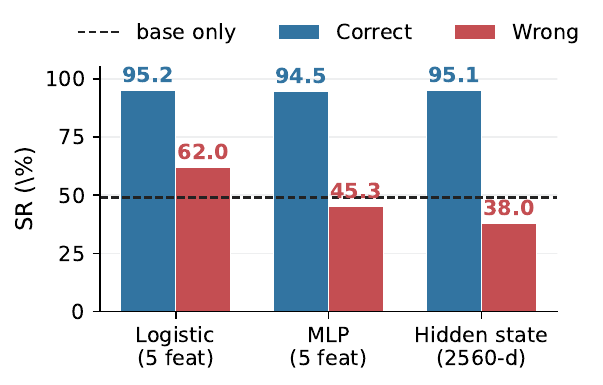}
\vspace{-0.1in}
\caption{Gate complexity ablation on HotpotQA (Qwen3-4B). With
correct direction all gates reach $\approx$95\%; with
wrong direction.}
\label{fig:capacity}
\vspace{-10pt}
\end{wrapfigure}
\vspace{-0.025in}
\textbf{Gate capacity does not substitute for direction.}
A natural objection is that DIAL's gains come from gate simplicity rather than learned direction. With correct
direction, a logistic gate, a 2-layer MLP, and a hidden-state probe (2560-d) all reach $\approx$95\% SR on HotpotQA: capacity adds $<$1\% (Figure~\ref{fig:capacity}). With wrong direction, the ordering inverts: the MLP drops to 45.3\%, \emph{below} the wrong-direction logistic at 62.0\%. Direction consistently outweighs capacity, which is why DIAL's minimal linear gate suffices.

\vspace{-0.025in}
\textbf{Trigger rate adapts to rollout headroom.}
In rollout-harmful environments, baselines that fire aggressively pay a price beyond efficiency: on Plancraft
($\Delta{=}{-}7.0$), s1\_budget keeps triggering and falls below \emph{base\_only} (16.3\% vs.\ 29.8\%, see
Appendix~\ref{app:bounds-comparison}); more compute makes performance \emph{worse} than no gate at all. DIAL avoids this trap without ever observing the headroom: its trigger rate adapts from 73\% on rollout-safe TWExpress, where almost any
trigger helps, down to $<$20\% on Plancraft, with a within-episode decay from 49\% at step~0 to under 20\% in late
steps as the gate learns that further rollouts only add harm. This adaptive magnitude emerges from the per-environment signed weights alone: direction discovery already encodes whether the environment rewards or penalizes additional computation. The full per-step analysis is in Appendix~\ref{app:trigger-rate}.

\vspace{-0.025in}
\textbf{Why fewer rollouts can improve SR.}
The optimizer is not a monotone improvement operator. In intervention-unsuitable states, rollout evaluation starts from the same unreliable context as the base policy and can amplify misleading evidence, invalid intermediate states, or compounding trajectory errors. Avoiding such negative-utility interventions therefore improves accuracy, not merely efficiency. DIAL's gain in low-headroom regimes comes from learning when \emph{not} to invoke the optimizer, while still triggering in decision-difficult states where rollouts expose useful alternatives.

\vspace{-0.15in}
\subsection{Two-Source Model Verification}
\vspace{-0.1in}
\label{sec:theory-verification}
\begin{wrapfigure}{r}{0.5\textwidth}
	\vspace{-33pt}
	\centering
	\includegraphics[width=0.43\textwidth]{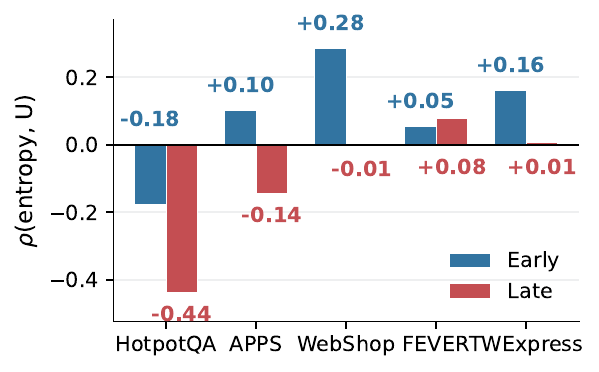}
	\vspace{-0.1in}
	\caption{P1: $\rho$(entropy, $U$) shifts from early (blue,
		step $\leq$ median) to late (red); bars show Spearman $\rho$.
		Plancraft omitted ($|\rho|{=}0.016$).}
	\label{fig:temporal-shift}
	\vspace{-10pt}
\end{wrapfigure}
The empirical landscape (\S\ref{sec:empirical-landscape}) shows
direction varies but does not pin down the Two-Source Model
(\S\ref{sec:toy-model}) as the cause. We test the model with
three complementary experiments that each rule out a different
class of alternative explanation: within-episode temporal
dynamics (P1, observational), within-task causal manipulation
of the information mixture (P2, interventional), and
cross-environment alignment of the dominant signal with the
dominant state type (P3, population-level).

\textbf{P1: Temporal dynamics.}
The Two-Source Model predicts that $\rho$ should decrease over the episode: as the agent gathers information, Type~D states are resolved and late-step uncertainty increasingly reflects residual Type~I states.
Figure~\ref{fig:temporal-shift} supports this pattern in environments with strong temporal signal. In HotpotQA, $\rho$ shifts from $-0.176$ (early) to $-0.437$ (late). In
WebShop, the strong positive early-phase signal ($\rho{=}{+}0.285$) vanishes in the late phase. In APPS, $\rho$ shifts from non-significant positive ($+$0.102) to weak negative ($-$0.144). Full results are in
Appendix~\ref{app:temporal}.

%\vspace{-0.05in}
\begin{wraptable}{r}{0.5\textwidth}
\vspace{-12pt}
\caption{Controlled information manipulation on HotpotQA
(Qwen3-4B): availability shifts which signal dominates,
providing causal evidence for the Two-Source Model.}
\vspace{-0.05in}
\label{tab:controlled}
\centering\footnotesize
\setlength{\tabcolsep}{4pt}
\begin{tabular}{lcccc}
\toprule
Condition & Base SR & $\rho$(ent.) & $\rho$(step) & DIAL SR \\
\midrule
Original  & 49.0\% & $-$0.041 & $-$0.494 & 95.2\% \\
InfoPoor  & 44.0\% & $+$0.119 & $\mathbf{-0.608}$ & 94.0\% \\
InfoRich  & 82.0\% & $\mathbf{+0.311}$ & $-$0.147 & 88.3\% \\
\bottomrule
\end{tabular}
\vspace{-10pt}
\end{wraptable}
\textbf{P2: Controlled information manipulation.} P1 verifies a within-episode dynamic; we now intervene on $p_I$ directly for \emph{causal} evidence. We construct two
interventions on HotpotQA: \textbf{InfoPoor} restricts retrievable evidence to a small passage subset (increasing the Type~I mixture), and \textbf{InfoRich} prepends gold-standard evidence to the prompt (reducing one major source of intervention-unsuitability, namely information poverty, and shifting states toward Type~D). Table~\ref{tab:controlled}
reports the result: restricting evidence makes \texttt{step\_count} dominant ($\rho{=}{-}0.608$) while entropy
is weak; injecting gold evidence makes entropy dominant ($\rho{=}{+}0.311$) while \texttt{step\_count} weakens. As the
only test in this section that intervenes on the latent mixture rather than observing it, P2 confirms that signal identity is a function of information structure.
%, not an
%environment-specific constant.
%\vspace{-0.05in}

\begin{wraptable}{r}{0.6\textwidth}
\vspace{-12pt}
\caption{Strongest signal per environment aligns with the
dominant regime: information-sufficiency proxies in Type~I,
decision-complexity in Type~D.}
\vspace{-0.05in}
\label{tab:signal-identity}
\centering\footnotesize
\setlength{\tabcolsep}{3pt}
\begin{tabular}{lllcl}
\toprule
\textbf{Env} & \textbf{Regime} & \textbf{Strongest Signal} & $|\rho|$ & \textbf{Interpretation} \\
\midrule
HotpotQA   & Type~I & step\_count        & 0.494 & Info accumulation \\
FEVER      & Type~I & step\_count        & 0.619 & Info accumulation \\
TWExpress  & Type~I & step\_count        & 0.477 & Exploration depth \\
WebShop    & Mixed  & num\_avail\_actions & 0.444 & Decision-space size \\
APPS       & Type~D & step\_count        & 0.339 & Debugging depth \\
Plancraft  & Weak   & has\_output        & 0.162 & Rollout harmful \\
\bottomrule
\end{tabular}
\vspace{-10pt}
\end{wraptable}
\textbf{P3: Signal identity alignment.}
P2 manipulated information within one task; P3 asks whether the
predicted alignment also holds \emph{across} tasks. The
Two-Source Model further predicts that the strongest signal in
each environment should reflect its dominant state type:
information-sufficiency proxies in Type~I, decision-complexity
proxies in Type~D. Table~\ref{tab:signal-identity} confirms this: \texttt{step\_count} dominates in Type~I environments,
while \texttt{num\_available\_actions} and other choice-related features dominate in WebShop. The LASSO coefficient matrix and its interpretation as Two-Source proxies are in Appendix~\ref{app:feature-selection}.

The three tests are observational, interventional, and
population-level respectively, and rule out distinct confounds.
Cross-setting $\rho$ consistency at fixed backbone is reported
in Appendix~\ref{app:cross-setting}, and multi-backbone
verification in Appendix~\ref{app:multi-backbone} provides
additional robustness checks.

% Tables tab:controlled and tab:signal-identity moved inline as wraptables in P2 and P3.

\section{Limitations}
\label{sec:discussion}

FEVER is DIAL's weakest setting: useful rollouts concentrate in the earliest steps, which random exploration may miss, and on one backbone a fixed-direction baseline matches or slightly exceeds DIAL on the SR-cost frontier. Curriculum-based exploration is a natural extension. DIAL incurs a one-time exploration cost per (environment, backbone), amortized over subsequent inference; in exchange, it removes the need for manual direction tuning. We exclude end-to-end RL gating controllers, whose reward-cost objective and online training can confound direction-attribution.

% ============================================================
% 7. CONCLUSION
% ============================================================

\section{Conclusion}

Adaptive gating for LLM agents is not difficulty estimation: it requires learning when the available rollout optimizer has positive counterfactual utility. We show that common fixed-direction gates fail in heterogeneous agent settings: the signal--utility relationship can reverse across environments and model backbones, and wrong-direction gates can worsen performance by selecting harmful states. We explain this reversal through a two-source model that distinguishes compute need from compute suitability: intervention-unsuitable and decision-difficult states give the same signal opposite meanings. The proposed DIAL learns the utility direction per (environment, backbone) setting from interaction data, using a sparse feature-based gate to achieve a stronger overall success-cost trade-off than fixed-direction baselines across six environments and three backbones.

% ============================================================
% REFERENCES
% ============================================================
\newpage
\bibliographystyle{unsrtnat}
\bibliography{references}
\newpage
% ============================================================
% APPENDIX
% ============================================================
\appendix

% ============================================================
% A. EXTENDED RELATED WORK
% ============================================================
\section{Extended Related Work}
\label{app:related}

\subsection{Adaptive Compute: Reasoning vs.\ Agent Settings}

The adaptive test-time compute literature spans two distinct
settings. In the \emph{reasoning} setting, methods optimize
single-turn chain-of-thought generation on homogeneous benchmarks
(MATH, GSM8K, GPQA).
Signal-based approaches use token-level confidence or entropy to
decide when to extend
thinking~\cite{seag2025, cats2025, corefine2026,
thinkjustenough2025, s1_2025, han2025tokenbudget}.
RL-based methods learn think/no-think policies through reinforcement
learning~\cite{adaptthink2025, thinkless2025, aggarwal2025l1}.
Hidden-state probing methods train lightweight classifiers on
internal representations to estimate task
difficulty~\cite{diffadapt2025, zhu2025llmalreadyknows},
sometimes assuming a universal U-shaped entropy pattern
that our cross-(environment, backbone) evidence disproves.
Because reasoning benchmarks share similar task semantics, signal
direction happens to be consistent across tasks, and the
fixed-direction assumption is never tested.

In the \emph{agent} setting, methods operate in multi-step
interactive environments with external optimizers. Vote-based
methods~\cite{catts2026, corefine2026} sample multiple candidate
actions and trigger evaluation when disagreement is high.
Verbalized-uncertainty methods~\cite{auq2026} use self-reported
confidence for agent reflection but assume fixed direction.
Episode-level approaches~\cite{paglieri2025learning} learn when
to invoke a planner, and agentic RL methods~\cite{arpo2025} use
entropy-based adaptive rollout at tool-call steps. All agent-setting
methods inherit the fixed-direction assumption from the reasoning
literature without verifying it. Our work is the first to
systematically study signal semantics across heterogeneous agent
environments \emph{and backbones}, where we find that the
direction reverses along both axes.

\subsection{Adaptive Computation in Neural Networks}

The idea of adapting computation to input difficulty predates LLMs.
Adaptive Computation Time~\cite{graves2016act} allows RNNs to learn
how many computational steps to take per input. Early exit
methods~\cite{teerapittayanon2016branchynet, elhoushi2024layerskip}
attach classifiers at intermediate transformer layers and exit early
for easy inputs. Mixture-of-Experts
architectures~\cite{shazeer2017moe, raposo2024mixtureofdepths}
activate only a subset of parameters per token, adapting capacity
rather than depth. These methods adapt computation within a single
forward pass. Our work operates at a different granularity: we
adapt whether to invoke an external optimizer at each decision step,
and the key challenge is not input difficulty but the
\emph{direction} of the signal-utility relationship.

\subsection{Confidence Calibration vs.\ Utility Calibration}
\label{app:calibration}

\paragraph{Classical confidence calibration.}
\citet{kadavath2022knowwhat} show that larger models are
well-calibrated on multiple-choice tasks and can estimate P(True)
for their own outputs.
\citet{tao2025revisiting} and \citet{heo2025llmuncertainty} find
that calibration quality varies substantially across task types,
providing converging evidence that uncertainty semantics are
domain-dependent.
\citet{yadkori2024believe} propose decoupling epistemic and
aleatoric uncertainty in LLM outputs, a distinction conceptually
related to our Type~I/Type~D decomposition.

\paragraph{From confidence-to-correctness to confidence-to-utility.}
These works ask whether confidence is calibrated to the
\emph{correctness} of an action. Adaptive gating, however,
requires a different alignment: whether a confidence-derived
signal predicts the \emph{value of additional computation}. We
call this latter target \emph{utility calibration}. The
distinction matters because even a perfectly
correctness-calibrated confidence may be insufficient for
gating: a gate triggers only when extra computation improves
the outcome, which depends on the environment and the
optimizer, not just on whether the current action is right.

\paragraph{Our contribution.}
We extend these findings in two directions. First, we make the
target distinction explicit: gating is a utility-calibration
problem, not a correctness-calibration one. Second, we show
that this strengthens the domain-dependence picture
empirically: not only does calibration \emph{quality} vary
across tasks, but the \emph{sign} of the signal-to-utility
relationship reverses across (environment, backbone) pairs, so
a gate using correctness-calibrated confidence under the wrong
utility direction will systematically select harmful states.

\subsection{Metareasoning and Value of Computation}

Rational metareasoning~\cite{russell1991right,
horvitz1987reasoning} formalizes the value of computation (VOC)
as a guide for allocating cognitive resources. The classical
framework assumes VOC~$\geq 0$ because an agent can always ignore
computation results. \citet{lieder2020resource} generalize this to
resource-rational analysis, modeling cognition as optimal use of
limited computational resources.
\citet{desabbata2024metareasoning} apply metareasoning to LLMs,
analyzing when to extend chain-of-thought reasoning.
Our work shows that the classical VOC~$\geq 0$ assumption is
violated in the agent setting; we formalize this in
Appendix~\ref{app:voc}.

\subsection{Simpson's Paradox in Machine Learning}
\label{app:simpson}

Simpson's paradox~\cite{simpson1951interpretation,
DBLP:books/acm/22/Pearl22n} occurs when a trend present in subgroups
reverses upon aggregation. In machine learning, this phenomenon has
been documented in fairness~\cite{kdd2023simpson}, causal
inference, and treatment effect estimation. Our two-source model
identifies a new instance: within Type~I states, entropy negatively
predicts utility; within Type~D states, entropy positively predicts
utility; but the marginal correlation depends on the mixture
proportion $p_I$, which is itself an (environment, backbone)-indexed
quantity, since a state's status as Type~I or Type~D depends on what the
agent already encodes, so the same environment can sit on opposite
sides of the reversal threshold $p_I^*$ under different LLMs. This
makes direction reversal a structural consequence of aggregating
heterogeneous uncertainty sources, and predicts both the
across-environment and across-backbone reversals reported in
Table~\ref{tab:signal-discovery}.

\subsection{Search-Based Optimizers}

A growing body of work targets \emph{what} the test-time compute
procedure does, rather than \emph{when} it should be invoked.
Tree-search and lookahead methods organize multi-step
exploration through value propagation and AlphaZero-style
rollouts~\cite{yao2023tree,zhou2024lats,hao2023reasoning,wang2026flare,feng2024alphazero,inoue2025abmcts}.
Repeated independent sampling with majority aggregation
exploits smooth coverage gains as the sample budget
grows~\cite{wang2023selfconsistency,brown2024monkeys}, while
verifiers trained over intermediate reasoning steps rerank
candidates or steer
search~\cite{lightman2024lets,wang2024mathshepherd}. At a
higher level, compute-allocation analyses and reasoning models
study how a fixed budget should be distributed across these
mechanisms or absorbed into long latent
traces~\cite{snell2024scaling,deepseek2025r1}.
Closer to our setting, agent-specific work adapts these
mechanisms to interactive, multi-step environments by applying
tree search to web tasks~\cite{koh2024tree}, coupling MCTS with
preference learning over agent
trajectories~\cite{putta2024agentq}, and characterizing scaling
behavior for agent
rollouts~\cite{DBLP:journals/corr/abs-2506-12928,huang2026evolverouter,DBLP:journals/corr/abs-2510-05445,ma2025autodata}.
All of this work specifies the rollout procedure that runs at a
triggered step. Our contribution is orthogonal: a per-step gate
decides \emph{whether} such a procedure should run, and the gate
composes with any of the above as a drop-in test-time optimizer.

% ============================================================
% B. THEORETICAL MATERIAL
% ============================================================
\section{Theoretical Material}
\label{app:theory}

This appendix collects the theoretical material referenced in
Section~\ref{sec:signal-landscape}: the formal statement and proof
of the necessity-of-direction-discovery result
(\S\ref{app:proof-necessity}), the wrong-direction damage
quantification (\S\ref{app:wrong-direction}), and the VOC
negativity analysis (\S\ref{app:voc}).

\subsection{Insufficiency of Uncertainty Signals: Statement and Proof}
\label{app:proof-necessity}

We formalize the central theoretical claim of the paper: any
adaptive-compute policy that depends on an uncertainty signal
$\sigma$ alone cannot deliver non-negative value of computation
across heterogeneous (environment, backbone) settings.

\begin{proposition}[Insufficiency of Uncertainty Signals]
\label{prop:necessity}
Let $\sigma : \mathcal{S} \to \mathbb{R}$ be any uncertainty
signal (e.g., token entropy, calibrated confidence, vote
disagreement, verbalized uncertainty). Let
$\mathrm{VOC}(s) =
 \mathbb{E}[R(\tau)\mid a{=}T(s)] -
 \mathbb{E}[R(\tau)\mid a{=}\pi(s)]$
denote the value of computation at state $s$.
Then $\sigma$ is \emph{not a sufficient statistic} for
$\mathrm{VOC}$ in the following sense: there exist (environment,
backbone) settings $\mathcal{E}_1, \mathcal{E}_2$ and states
$s_1 \in \mathcal{E}_1$, $s_2 \in \mathcal{E}_2$ with
$\sigma(s_1) = \sigma(s_2)$ such that
$\mathrm{VOC}(s_1) > 0$ and $\mathrm{VOC}(s_2) < 0$.
Consequently, no policy of the form
$\pi_\sigma(s) = g(\sigma(s))$ for any measurable
$g : \mathbb{R} \to \{0, 1\}$ can simultaneously satisfy
$\mathrm{SR}(\pi_\sigma, \mathcal{E}_1) \geq
 \mathrm{SR}(\mathrm{base}, \mathcal{E}_1)$ and
$\mathrm{SR}(\pi_\sigma, \mathcal{E}_2) \geq
 \mathrm{SR}(\mathrm{base}, \mathcal{E}_2)$.
\end{proposition}

The proposition strengthens the earlier ``no fixed-direction
gate works'' formulation in two ways. First, it rules out
\emph{any} measurable $g(\sigma)$, not just monotone or
threshold-based policies, so non-monotone or U-shaped
calibrators are also covered. Second, it identifies the
underlying obstruction as an information-theoretic property of
$\sigma$ itself, one that survives any post-processing of the
signal.

\begin{proof}
By the two-source decomposition of \S\ref{sec:toy-model},
$\mathrm{VOC}(s)$ depends on a latent type
$\tau(s) \in \{I, D\}$ that is not measurable from $\sigma(s)$
alone:
\[
  \mathrm{VOC}(s) \;=\;
  \begin{cases}
    -\alpha\,\sigma(s) + \varepsilon_I & \text{if } \tau(s) = I,\\
    +\beta\,\sigma(s) + \varepsilon_D & \text{if } \tau(s) = D,
  \end{cases}
\]
with $\alpha, \beta > 0$. Choose any signal value
$\sigma_0$ in the support of both Type~I and Type~D states
within their respective settings. Then there exist
$s_1 \in \mathcal{E}_1$ with $\tau(s_1) = D$ and
$\sigma(s_1) = \sigma_0$, giving
$\mathbb{E}[\mathrm{VOC}(s_1)\mid \sigma_0] = \beta \sigma_0 > 0$
for sufficiently large $\sigma_0$, and
$s_2 \in \mathcal{E}_2$ with $\tau(s_2) = I$ and
$\sigma(s_2) = \sigma_0$, giving
$\mathbb{E}[\mathrm{VOC}(s_2)\mid \sigma_0] =
-\alpha \sigma_0 < 0$. Hence the conditional VOC is not a
function of $\sigma$ alone, establishing insufficiency.

For the policy claim, fix any $g : \mathbb{R} \to \{0, 1\}$.
The expected one-step reward gain over base from playing
$\pi_\sigma$ is
$\mathbb{E}_{s \sim \mathcal{E}}[g(\sigma(s)) \cdot \mathrm{VOC}(s)]$.
Decomposing by latent type and writing $p_I(\mathcal{E})$ for
the Type~I fraction,
\begin{align*}
&\mathbb{E}_{\mathcal{E}}[g(\sigma) \cdot \mathrm{VOC}]
= -\alpha\, p_I(\mathcal{E})\, \mathbb{E}_I[g(\sigma)\sigma] \\
&\qquad + \beta\,(1 - p_I(\mathcal{E}))\, \mathbb{E}_D[g(\sigma)\sigma] + \mathrm{const}.
\end{align*}
Whenever $g$ is non-trivially correlated with $\sigma$ (which
any informative gate must be), the two terms have opposite
signs as $p_I$ varies. In particular, for any fixed $g$ there
exist $\mathcal{E}_1, \mathcal{E}_2$ with $p_I(\mathcal{E}_1)$
and $p_I(\mathcal{E}_2)$ on opposite sides of the reversal
threshold $p_I^* = \beta/(\alpha+\beta)$ of \S\ref{sec:toy-model}
such that the expected gain has opposite signs. The setting
where it is negative violates the SR constraint relative to
$\mathrm{base}$, completing the argument.
\end{proof}

\paragraph{Corollary (necessity of direction discovery).}
The original ``no fixed-direction gate works'' result is the
restriction of Proposition~\ref{prop:necessity} to threshold
policies $g(\sigma) = \mathbf{1}[d \cdot \sigma > \theta]$,
$d \in \{+1, -1\}$. In that case both cases of the proof above
collapse to: $g$'s sign of correlation with $\sigma$ is fixed at
$d$, but $\mathrm{VOC}$'s sign of correlation with $\sigma$
flips with $p_I$, so on the wrong side of $p_I^*$ the gate
triggers exactly when $\mathrm{VOC} < 0$ and incurs systematic
harm.

\paragraph{What recovers sufficiency.}
The proposition pinpoints what an adaptive-compute method must
condition on: any feature $\phi(s)$ that is informative about
the latent type $\tau(s)$ recovers a sufficient statistic for
$\mathrm{VOC}$, since
$\mathrm{VOC}(s) = f(\sigma(s), \tau(s))$ is by construction a
function of $(\sigma, \phi)$. DIAL's multi-feature gate
(\S\ref{sec:method}) is the minimal procedure that operationalizes
this recovery: it augments $\sigma$ with auxiliary universal
features that empirically correlate with $\tau$ (e.g.,
\texttt{step\_count} as a proxy for information accumulation,
\texttt{num\_available\_actions} as a proxy for decision
optionality) and lets the data identify which of them carry
information about $\tau$ in each (environment, backbone) setting.

\subsection{Wrong-Direction Damage: The Calibration-Worsens
  Mechanism}
\label{app:wrong-direction}

Table~\ref{tab:wrong-direction} is the empirical statement of
the headline claim: a gate with reversed weights is, by
construction, a more precisely calibrated \emph{harm-targeting}
mechanism, and the resulting SR collapse scales with signal
strength $|\rho^*|$. The decomposition predicts the relationship
quantitatively. By the wrong-direction lemma in
App~\ref{app:proof-necessity},
\[
  \mathrm{SR}(g_{-d^*}, \mathcal{E}) -
  \mathrm{SR}(\mathrm{base}, \mathcal{E})
  \;\approx\; -\,c \cdot |\rho^*(\mathcal{E})|,
\]
where $c$ is a constant that depends on the optimizer's harm
magnitude. Empirically this scaling is clean: HotpotQA
($|\rho^*|{=}0.494$) loses 37.0\%, FEVER ($|\rho^*|{=}0.619$)
loses 36.8\%, WebShop ($|\rho^*|{=}0.444$) loses 23.2\%,
while Plancraft ($|\rho^*|{=}0.016$) loses only 2.8\%; the
slope of damage in $|\rho^*|$ is the slope of ranking
sharpness, since sharper-ranking gates concentrate their firing
on higher-$|\sigma|$ states which, when the direction is wrong,
are exactly the harmful ones.

Two environments show only minimal damage under reversal, both
consistent with the same prediction (small $|\rho^*|$ or
rollout-safe optimizer):

\paragraph{TWExpress is rollout-safe ($\Delta{=}{+}31.8$).}
The optimizer rarely harms performance regardless of when it is
triggered, so both directions achieve near-perfect SR (99.0\%).

\paragraph{APPS shows a small negative $\Delta$SR ($-2.5\%$).}
This is consistent with its unstable signal direction across
backbones (Table~\ref{tab:signal-discovery}): Qwen3 shows
positive $\rho$, Llama-3.1 shows negative $\rho$, and Phi-3.5 is
non-significant. When the signal is this unstable, direction
reversal has minimal impact and small fluctuations can favor
either direction.

\subsection{VOC Non-Negativity Scope}
\label{app:voc}

In the classical VOC framework~\cite{russell1991right},
$\mathrm{VOC}(T, s) =
 \mathbb{E}[\max(V(a_T), V(a_{\mathrm{base}}))]
 - V(a_{\mathrm{base}}) \geq 0$
because the agent can \emph{disregard} the computation result if
unfavorable. Our setting violates this: when the agent uses an
evaluator/optimizer $T$, the evaluator's assessment \emph{is} the
agent's decision. The agent cannot separately evaluate the
evaluator's output because the evaluator \emph{is} the evaluation
mechanism.

Under evaluator-executor identity, VOC can be negative:
$\mathrm{VOC}(T, s) =
 \mathbb{E}[V(a_T(s))] - V(a_{\mathrm{base}}(s))$.
If $T$ systematically selects worse actions in Type~I states,
$\mathrm{VOC} < 0$ for those states. This explains why
wrong-direction gating is not merely wasteful
($\mathrm{VOC} = 0$) but actively harmful
($\mathrm{VOC} < 0$), and why \emph{sharper-ranking} gates
fire more precisely on these $\mathrm{VOC} < 0$ states, so
ranking improvements amplify rather than mitigate the harm
(Table~\ref{tab:wrong-direction} and
App~\ref{app:wrong-direction}). The classical VOC framework's
$\mathrm{VOC} \geq 0$ guarantee fails for adaptive compute in
agent settings precisely because the agent cannot ``disregard''
the optimizer's output, and any policy that depends only on
$\sigma$ (which cannot identify Type~I states) will trigger
$T$ on a non-trivial fraction of $\mathrm{VOC} < 0$ states.

% ============================================================
% C. DIAL IMPLEMENTATION DETAILS
% ============================================================
\section{DIAL Implementation Details}
\label{app:dial-details}

This appendix expands on DIAL's implementation: hyperparameters
and training details (\S\ref{app:hyperparams-details}), universal
feature definitions (\S\ref{app:universal-features}), the optional
LLM feature layer (\S\ref{app:llm-feature-layer}), online
adaptation for drifting environments (\S\ref{app:online-adaptation}),
and the LLM prompt template (\S\ref{app:llm-prompts}).

\begin{algorithm}[h]
	\caption{DIAL: Direction-Informed Adaptive Learning}
	\label{alg:dial}
	\small
	\begin{algorithmic}[1]
		\Require Environment $\mathcal{E}$, agent policy $\pi$, optimizer
		$T$, universal feature set $\phi_{\text{univ}}$, LLM
		$\mathcal{M}$, exploration budget $N_{\text{explore}}$,
		exploration rate $\varepsilon_{\text{explore}}$,
		gate threshold $\tau$
		\Ensure Deployed gate $g: \mathcal{S} \to \{0, 1\}$

		\Statex \textbf{// Phase 1: Explore}
		\State $\mathcal{D} \gets \emptyset$
		\For{episode $= 1, \dots, N_{\text{explore}}$}
		  \State Run $\pi$ in $\mathcal{E}$; at each step $t$, invoke $T$
		  with probability $\varepsilon_{\text{explore}}$
		  independent of state
		  \State Record $(s_t, U_t)$ pairs into $\mathcal{D}$
		\EndFor

		\Statex \textbf{// Phase 2: Reason}
		\State $(p_{\mathcal{E}}, \mathcal{H}) \gets
		\mathcal{M}(\mathrm{summary}(\mathcal{D}))$
		\Comment{env profile + feature hypotheses}
		\State $\phi_{\text{cand}} \gets \phi_{\text{univ}} \cup
		\mathrm{extract}(\mathcal{H})$
		\State Recompute $\phi_{\text{cand}}(s)$ for all
		$s \in \mathcal{D}$

		\Statex \textbf{// Phase 3: Learn}
		\State Standardize $\phi_{\text{cand}}$ on $\mathcal{D}$
		\State Select $C$ via 5-fold CV on held-out log-loss
		\State Fit $\ell_1$-regularized logistic regression:
		\begin{equation}
		\mathbf{w}^*, b^* = \argmin_{\mathbf{w}, b}\
		\sum_{(\phi, U) \in \mathcal{D}}
		\ell\bigl(\sigma(\mathbf{w}^\top \phi + b),\, U\bigr)
		+ \tfrac{1}{C} \|\mathbf{w}\|_1
		\label{eq:lasso-full}
		\end{equation}
		\State $\phi \gets$ features with non-zero $w_i^*$
		\State $g(s) \gets
		\mathbf{1}[\sigma(\mathbf{w}^{*\top} \phi(s) + b^*) > \tau]$

		\State \Return $g$
	\end{algorithmic}
\end{algorithm}

\subsection{Hyperparameters and Training Details}
\label{app:hyperparams-details}

Phase~1 runs for $N_{\text{explore}} = 50$ episodes with
$\varepsilon_{\text{explore}} = 0.5$. Before fitting in Phase~3,
all features in $\phi_{\text{cand}}$ are standardized to zero mean
and unit variance on $\mathcal{D}$, ensuring that the interpretable
weight signs discussed in Section~\ref{sec:learn} are not
confounded by scale differences across heterogeneous features.
We induce sparsity using the $\ell_1$-regularized logistic
regression of Eq.~\eqref{eq:lasso-full}, with $\ell$ the binary
cross-entropy loss and regularization strength $C$ selected by
5-fold cross-validation over
$C \in \{0.01, 0.03, 0.1, 0.3, 1.0, 3.0, 10.0\}$, optimizing
held-out log-loss. Features with zero coefficient after fitting
are dropped from the deployed gate. The decision threshold $\tau$
in Eq.~\eqref{eq:gate} is fixed to $0.5$ for all environments.

\subsection{Universal Feature Definitions}
\label{app:universal-features}

Table~\ref{tab:universal-features} defines the universal features
available in every environment.

\begin{table}[H]
\centering\small
\caption{Universal feature definitions. All features are
computed identically across environments; signals an
environment does not naturally expose default to zero.}
\label{tab:universal-features}
\begin{tabular}{ll}
\toprule
Feature & Definition \\
\midrule
\texttt{step\_count}            & Current step index $t$ in the episode \\
\texttt{token\_entropy}         & Entropy of the next-token distribution \\
\texttt{evidence\_count}        & Count of evidence/information items collected (QA/search envs) \\
\texttt{num\_available\_actions}& Count of currently-available actions (navigation envs) \\
\texttt{is\_finish}             & Whether the proposed action is the episode-terminating ``finish'' \\
\bottomrule
\end{tabular}
\end{table}

\paragraph{Additional feature categories used only in the
regularizer ablation.} The expanded candidate pool of
Appendix~\ref{app:regularizer-ablation} additionally includes
three \emph{derived} features (\texttt{step\_ratio}~$=$~step\_count$/$max\_steps,
\texttt{entropy\_sq}~$=$~token\_entropy$^{2}$, and
\texttt{step\_x\_entropy}~$=$~step\_count~$\times$~token\_entropy),
20 \emph{hidden-state PCA} components extracted from the model's
internal representations, and three \emph{auto-extracted}
features parsed from the state text (\texttt{state\_length},
\texttt{num\_numbers}, and \texttt{has\_error}). The main DIAL
configuration uses only the universal features above plus the
LLM-generated layer (Appendix~\ref{app:llm-feature-layer}).

\subsection{LLM Feature Layer}
\label{app:llm-feature-layer}

For environments where universal features may be inadequate,
DIAL's optional LLM feature layer automates the proposal of
task-specific signals. The LLM receives a structured summary of
the exploration dataset $\mathcal{D}$, including: (1)~aggregate
statistics (total episodes, steps, optimizer trigger rate, overall
positive-utility fraction), (2)~per-step breakdowns of trigger
rate and utility by step index, and (3)~representative examples
of states where the optimizer was triggered, paired with their
utility labels (positive vs.\ negative).

Given this summary, the LLM is prompted to analyze patterns that
distinguish positive-utility from negative-utility states and to
produce a Python function that extracts exactly 5 task-specific
features from the raw state text. Each feature must return a
float value using only the Python standard library. The full
prompt template is shown in Figure~\ref{fig:llm-prompt-template}
(\S\ref{app:llm-prompts}).

The LLM-proposed features $\phi_{\text{LLM}}(s)$ are combined
with the universal features to form the candidate pool
$\phi_{\text{cand}}(s) = \phi_{\text{univ}}(s) \cup
\phi_{\text{LLM}}(s)$. The $\ell_1$-regularized gate
(\S\ref{sec:learn}) then selects or discards each feature based
on its predictive value for rollout utility. This two-stage
design (LLM proposes, LASSO disposes) ensures that uninformative
LLM features are automatically filtered out. As shown in
Table~\ref{tab:llm-features}, environments with richer
task-specific structure (WebShop, TWExpress) retain more LLM
features, while Plancraft retains none.

\subsection{Online Adaptation}
\label{app:online-adaptation}

For drifting environments, DIAL can optionally refit during
deployment using $\varepsilon$-greedy exploration
($\varepsilon: 0.1 \to 0$) and periodic retraining. During
deployment, $\varepsilon_{\text{deploy}}$ decays from $0.1$
to $0$ over the first 100 episodes, and the gate is refit every
$k = 30$ episodes using the accumulated $\varepsilon$-greedy
samples. Refitting reuses the same solver configuration as the
initial training, so no additional hyperparameter tuning is
required at deployment time.

This mode is \emph{not free}: $\varepsilon$-greedy overrides
ignore the current gate, so in Type-I-dominant environments a
fraction of overridden steps will trigger harmful rollouts. For stationary
environments, which are the setting of our main experiments,
we recommend skipping online adaptation and relying on the
offline-fitted gate.

\subsection{LLM Reasoning Prompts}
\label{app:llm-prompts}

Figure~\ref{fig:llm-prompt-template} shows the full LLM prompt
template. We use WebShop as the example, where LLM features
contribute most meaningfully.

\begin{figure}[h]
  \centering
  \begin{tcolorbox}[
      title=\textbf{LLM Feature Layer Prompt Template},
      colframe=PTGreenDark,
      colback=PTLightGreen,
      colbacktitle=PTGreenDark,
      coltitle=white,
      fonttitle=\bfseries,
      boxrule=0.9pt,
      arc=0pt,
      left=7pt,right=7pt,top=7pt,bottom=7pt,
      width=\columnwidth,
      enhanced
    ]
    \ttfamily\footnotesize
    You are an AI agent that has been exploring an interactive environment. During exploration, you sometimes performed ``rollouts'' (additional computation) to improve your decisions. Sometimes the rollout was useful (positive utility), sometimes it was not.\\[0.6ex]
    Your task: Analyze the patterns in your experience and write a Python function that extracts features from the state text that could predict whether a rollout would be useful.\\[0.6ex]
    \textnormal{\textit{[Exploration statistics and representative examples provided]}}\\[0.6ex]
    Requirements:\\
    - Return a dict with string keys and float values.\\
    - Use only Python standard library.\\
    - Extract exactly 5 features based on the patterns you observe.\\
    - Focus on features that DIFFER between useful and not-useful rollout cases.
  \end{tcolorbox}
  \caption{Prompt template used in the LLM feature layer (Phase~2) to propose task-specific features from exploration data.}
  \label{fig:llm-prompt-template}
\end{figure}

% ============================================================
% D. EXPERIMENTAL SETUP
% ============================================================
\section{Experimental Setup}
\label{app:experiments}

\subsection{Environment Specifications}
\label{app:environment}

We evaluate on 6 environments spanning diverse agent tasks.
Table~\ref{tab:env-setup} summarizes their base performance,
optimizer headroom, and optimizer type.

\textbf{HotpotQA~\cite{yang2018hotpotqa}.}
A multi-hop question answering task where the agent must retrieve
and reason over multiple Wikipedia paragraphs to answer complex
questions. The agent iteratively issues search queries and reads
returned passages before producing a final answer.

\textbf{WebShop~\cite{yao2022webshop}.}
A web navigation task where the agent shops for a product matching
a natural language instruction. The agent navigates a simulated
e-commerce website by searching, clicking product links, and
selecting attributes (size, color) before purchasing.

\textbf{FEVER~\cite{thorne2018fever}.}
A fact verification task where the agent must classify a claim as
\emph{Supported}, \emph{Refuted}, or \emph{Not Enough Info} based
on evidence retrieved from Wikipedia. The agent searches for
relevant passages and reasons over retrieved evidence to produce
a verdict.

\textbf{APPS~\cite{hendrycks2021apps}.}
A code generation task where the agent solves competitive
programming problems. Given a problem description with input/output
examples, the agent generates a Python solution and can iteratively
debug based on test case feedback.

\textbf{TWExpress~\cite{jansen2023twexpress}.}
A text-based game where the agent navigates a simulated environment
to complete household tasks (e.g., finding and heating objects).
The agent issues natural language commands (go, open, take, use)
and observes textual descriptions of the resulting state.

\textbf{Plancraft~\cite{dagan2024plancraft}.}
A crafting planning task where the agent must produce a target item
by executing a sequence of crafting recipes in a Minecraft-inspired
environment. The agent selects recipes and manages inventory to
reach the goal item.

These environments cover the full range of the Two-Source Model:
FEVER and HotpotQA are dominated by Type~I states (information
poverty), WebShop and APPS are dominated by Type~D states
(decision difficulty), TWExpress is rollout-safe
($\Delta{=}{+}31.8$), and Plancraft is rollout-harmful
($\Delta{=}{-}7.0$).

All three optimizer types rank candidates by environment-provided
reward; they differ along two axes: what is varied (single action /
full program / full action sequence) and how the rollout is
aggregated (mean return of truncated rollouts / test-case pass /
full-episode reward).
\begin{itemize}[leftmargin=*,nosep]
\item \textbf{Per-action evaluation} (HotpotQA, FEVER, TWExpress,
  Plancraft). At the gated step, take the top-$K{=}5$ candidate
  actions proposed by the base policy. For each candidate $a$, fork
  the environment snapshot, execute $a$, and continue for $H{-}1{=}2$
  further steps with the same LLM proposer ($\tau{=}0.7$), giving a
  truncated rollout of length $H{=}3$. Repeat $N{=}5$ independent
  rollouts per candidate and set
  $V(a) = \tfrac{1}{N}\sum_{n=1}^{N} R_n(a)$, where $R_n(a)$ is the
  env-provided score for the $n$-th truncated trajectory (task-specific:
  e.g., F1 for QA, episode return for games). Execute
  $a^{*} = \arg\max_{a} V(a)$.
\item \textbf{$K$-variant sampling} (APPS). Sample $K{=}3$ complete
  code solutions and execute the first one that passes the provided
  test cases (objective check via the interpreter).
\item \textbf{LLM-Propose-$K$} (WebShop). The LLM proposes $K{=}3$
  full action sequences; each is simulated end-to-end in the
  environment, and the trajectory with the highest observed
  episodic reward is executed.
\end{itemize}
\paragraph{Utility estimation.}
In all environments, the utility label $U_t$ used for gate
training (\S\ref{sec:explore}) is computed by paired
counterfactual evaluation from the same state snapshot.
For per-action evaluation environments, the base policy's
proposed action $a_B$ is included in the top-$K$ candidate set
and evaluated with the same $N{\times}H$ rollout protocol as all
other candidates; utility is
$U_t = \mathbf{1}[V(a^*) > V(a_B)]$, where
$V(a) = \tfrac{1}{N}\sum_{n=1}^{N} R_n(a)$.
For $K$-variant sampling (APPS), the base solution is evaluated
on the same test suite as the $K$ variants, and
$U_t = \mathbf{1}[\max_k \text{pass\_rate}_k > \text{pass\_rate}_{\text{base}}]$.
Because all candidates are evaluated from identical initial
conditions, this is a counterfactual estimate rather than an
on-policy observation. The truncation horizon $H$ introduces
approximation bias relative to full-episode returns; we find
results stable for $H \in [3, 10]$.

In all three cases, the optimizer itself has access to the env's
scoring function; our method is about learning \emph{when} this
optimizer is worth invoking, not how it scores internally.

\begin{table}[h!]
\caption{Environment setup. Base/Always: SR without/with optimizer.
$\Delta$: rollout headroom.}
\label{tab:env-setup}
\centering\small
\begin{tabular}{lcccl}
\toprule
Environment & Base SR & Always SR & $\Delta$ & Optimizer $T$ \\
\midrule
HotpotQA       & 49.0\% & 97.0\% & +48.0   & Per-action eval \\
WebShop        &  7.2\% & 43.0\% & +35.8   & LLM-Propose-$K$ \\
FEVER          & 37.0\% & 99.8\% & +62.8   & Per-action eval \\
TWExpress      & 67.5\% & 99.3\% & +31.8   & Per-action eval \\
Plancraft      & 29.8\% & 22.8\% & $-$7.0  & Per-action eval \\
APPS           & 60.5\% & 79.5\% & +19.0   & $K$-variant sampling \\
\bottomrule
\end{tabular}
\end{table}

\subsection{Detailed Baseline Comparison}
\label{app:method-comparison}

Table~\ref{tab:baseline-detail} provides a detailed comparison of
all baselines evaluated in this paper. All six methods share a
common assumption: a fixed mapping from signal value to compute
need.

\begin{table*}[h!]
\caption{Detailed comparison of baselines. All methods assume a
fixed signal-utility direction.
$^\dagger$Requires a separate labeled dataset for calibration.}
\label{tab:baseline-detail}
\centering\small
\renewcommand{\arraystretch}{1.15}
\begin{tabularx}{\textwidth}{l W W N W}
\toprule
Method & Signal & Direction & Granularity & Mechanism \\
\midrule
CATTS~\cite{catts2026} & Vote entropy + margin
  & High disagreement $\to$ trigger
  & Per-step & $K{=}5$ forward passes, arbiter on disagreement \\
SEAG$^\dagger$~\cite{seag2025} & Mean token confidence
  & Low conf.\ $\to$ search
  & Per-problem & Platt scaling, search depth allocation \\
CaTS$^\dagger$~\cite{cats2025} & Calibrated confidence
  & Low conf.\ $\to$ generate
  & Per-attempt & Self-calibration, Best-of-N early stopping \\
CoRefine$^\dagger$~\cite{corefine2026} & Conv1D confidence
  & Low conf.\ $\to$ refine
  & Per-problem & Halt / re-examine / redirect controller \\
s1\_budget~\cite{s1_2025} & Token budget
  & Fixed budget $\to$ stop
  & Per-problem & Budget-constrained generation \\
AUQ~\cite{auq2026} & Verbalized uncertainty
  & High uncert.\ $\to$ reflect
  & Per-step & Uncertainty-aware memory + reflection \\
\midrule
\textbf{DIAL} & \textbf{Multi (auto)} & \textbf{Learned}
  & \textbf{Per-step}
  & Explore + reason + direction learning \\
\bottomrule
\end{tabularx}
\end{table*}

\subsection{Computational Cost Breakdown}
\label{app:cost}

Table~\ref{tab:cost-components} breaks down the per-method cost
structure. Cost is reported as $\times$base: total
\emph{deployment} tokens normalized by base\_only tokens per
episode. Calibration or exploration overhead is excluded for all
methods, ensuring a fair comparison of runtime efficiency. This
metric captures hidden per-step overhead that raw rollout counts
miss: CATTS incurs $K{=}5$ forward passes every step regardless
of the gating decision, and AUQ issues a confidence query every
step. Consistent with the cost accounting of \S\ref{sec:method},
DIAL incurs a one-time exploration cost of
$N_{\text{explore}}{=}50$ episodes per (environment, backbone)
deployment and sub-second offline gate fitting; at deployment the
gate adds zero per-step inference cost beyond the feature
evaluation and a single sigmoid: no extra LLM call, no $K$-way
voting, and no calibration query.

\begin{table}[H]
\centering
\caption{Computational cost breakdown per method.
Gate OH = additional LLM calls per step beyond the base proposer.
Cost ($\times$base) = total \emph{deployment} tokens normalized by
base\_only tokens per episode, averaged across 6 environments.
Calibration/exploration overhead is excluded for all methods.
$^\dagger$Requires separate calibration data.}
\label{tab:cost-components}
\small
%\resizebox{\columnwidth}{!}{%
\begin{tabular}{lclc}
\toprule
\textbf{Method} & \textbf{Gate OH/step} & \textbf{OH source} & \textbf{Avg $\times$b} \\
\midrule
base\_only        & 0          & ---                    & 1.00 \\
always\_trigger   & 0          & ---                    & 6.00 \\
\midrule
CATTS             & $+K{=}5$ calls & Voting (fwd pass)  & 7.98 \\
SEAG$^\dagger$    & 0          & Reads logprob          & 6.37 \\
CoRefine$^\dagger$& 0          & Reads entropy          & 6.33 \\
CaTS$^\dagger$    & 0          & Reads confidence       & 8.26 \\
AUQ               & $+1$ query & Confidence question    & 7.67 \\
s1\_budget        & 0          & ---                    & 6.35 \\
\midrule
\textbf{DIAL}     & \textbf{0} & \textbf{Learned gate} & \textbf{5.85} \\
\bottomrule
\end{tabular}%
%}
\end{table}

\subsection{DIAL vs.\ Bounds: Full Numbers}
\label{app:bounds-comparison}

Table~\ref{tab:bounds-comparison} reports DIAL's SR and Cost
against the two bounds (\emph{base\_only}, \emph{always\_trigger})
across all 6 environments on Qwen3-4B. In environments with
positive rollout headroom (HotpotQA, WebShop, TWExpress, APPS),
DIAL achieves SR close to always\_trigger at substantially lower
cost. In the rollout-harmful environment Plancraft
($\Delta{=}{-}7.0$), DIAL learns to suppress triggering, losing
only 6.5\% relative to base\_only while always\_trigger loses
7.0\%. FEVER remains the weakest case: DIAL improves over
base\_only (49.8\% vs.\ 37.0\%) but falls far short of
always\_trigger (59.8\%), a limitation analyzed in
\S\ref{sec:discussion}.

\begin{table}[h!]
\caption{DIAL vs.\ bounds on Qwen3-4B. $\Delta$: rollout headroom
(always$-$base).}
\label{tab:bounds-comparison}
\centering\small
\begin{tabular}{lcccccc}
\toprule
& \multicolumn{2}{c}{\textbf{base\_only}}
& \multicolumn{2}{c}{\textbf{always\_trigger}}
& \multicolumn{2}{c}{\textbf{DIAL}} \\
\cmidrule(lr){2-3}\cmidrule(lr){4-5}\cmidrule(lr){6-7}
\textbf{Env} & SR$\uparrow$ & Cost$\downarrow$ & SR$\uparrow$ & Cost$\downarrow$ & SR$\uparrow$ & Cost$\downarrow$ \\
\midrule
HotpotQA  & 49.0 & 1.00 & 97.0 & 10.63 & 95.2 & 8.02 \\
WebShop   &  7.2 & 1.00 & 43.0 &  5.56 & 43.8 & 2.50 \\
FEVER     & 37.0 & 1.00 & 59.8 &  17.87 & 49.8 & 16.51 \\
TWExpress & 67.5 & 1.00 & 99.3 &  2.15 & 99.0 & 1.81 \\
Plancraft & 29.8 & 1.00 & 22.8 &  6.13 & 23.3 & 3.63 \\
APPS      & 60.5 & 1.00 & 79.5 &  3.65 & 73.0 & 2.61 \\
\bottomrule
\end{tabular}
\end{table}

\subsection{Compute Resources}
\label{app:compute}

All experiments are run on NVIDIA A100 GPUs with 40\,GB of memory.
Each backbone (Qwen3-4B, Llama-3.1-8B-Instruct, Phi-3.5-mini-instruct)
fits on a single A100 in BF16, and we serve the model with vLLM for
batched inference. The full set of reported experiments covers
6 environments $\times$ 3 backbones $\times$ 3 seeds, plus the
DIAL exploration phase ($N_{\text{explore}}{=}50$ episodes per
(environment, backbone) deployment) and ablations.
We estimate the total compute used to produce the results in this
paper at approximately {1{,}900} A100-GPU-hours, of which
roughly {55\%} corresponds to evaluation rollouts (DIAL,
the two bounds, and six fixed-direction baselines across all
(environment, backbone, seed) combinations), {2\%} to DIAL
exploration, and the remainder to ablations (LLM-feature variants,
gate-capacity, wrong-direction reversal, regularizer,
controlled-reversal interventions, and supplementary multi-backbone
runs). The full research project, including preliminary experiments
and configurations that did not make it into the paper, used
roughly {2{,}600} A100-GPU-hours in total.

\subsection{Asset Licenses}
\label{app:licenses}

Table~\ref{tab:licenses} lists the public assets used in this
paper, together with their licenses. We use each asset within the
terms of its license; all datasets and model weights are obtained
from their official releases.

\begin{table}[h!]
\caption{Datasets and pre-trained models used in this paper.}
\label{tab:licenses}
\centering\small
\begin{tabular}{lll}
\toprule
\textbf{Asset} & \textbf{Type} & \textbf{License} \\
\midrule
HotpotQA~\cite{yang2018hotpotqa}              & Dataset       & CC BY-SA 4.0 \\
APPS~\cite{hendrycks2021apps}                 & Dataset       & MIT \\
WebShop~\cite{yao2022webshop}                 & Dataset+Env   & MIT (Princeton NLP) \\
FEVER~\cite{thorne2018fever}                  & Dataset       & CC BY-SA 3.0 \\
TextWorld Express~\cite{jansen2023twexpress}  & Environment   & Apache 2.0 \\
Plancraft~\cite{dagan2024plancraft}           & Environment   & MIT \\
Qwen3-4B~\cite{DBLP:journals/corr/abs-2505-09388}        & Model weights & Apache 2.0 \\
Llama-3.1-8B-Instruct~\cite{DBLP:journals/corr/abs-2407-21783} & Model weights & Llama 3.1 Community License \\
Phi-3.5-mini-instruct~\cite{DBLP:journals/corr/abs-2404-14219} & Model weights & MIT \\
\bottomrule
\end{tabular}
\end{table}

% ============================================================
% E. FULL RESULTS
% ============================================================
\section{Full Results Analysis}
\label{app:full-results}

Section~\ref{sec:main-results} (Table~\ref{tab:full-results})
reports SR and Cost ($\times$base) for all methods across 6
environments and 3 backbones. Cost includes all deployment LLM
calls (base proposer, gate overhead, rollouts) normalized by
base\_only cost. Bounds comparison is in
Table~\ref{tab:bounds-comparison} (\S\ref{app:bounds-comparison}).
This appendix elaborates per-backbone failure modes and cost
anomalies that the aggregated table does not surface.

\paragraph{Qwen3-4B.} DIAL achieves the highest SR across all
six environments. The $\times$base metric reveals hidden costs
of fixed-direction baselines: CATTS on APPS costs
6.00$\times$base (due to $K{=}5$ per-step voting) despite
barely triggering rollouts, while DIAL costs only
2.61$\times$ at 12.2\% higher SR. FEVER is the cost outlier:
DIAL still leads in SR (49.8\%) but at 16.51$\times$base, the
highest cost among all environments. This failure mode is
analyzed in \S\ref{sec:discussion}.

\paragraph{Phi-3.5-mini.} Fixed-direction baselines degrade
sharply on this backbone. On FEVER, all six baselines score at or
below 20\%, while DIAL reaches 20.5\%. On WebShop,
DIAL (57.3\%) exceeds the best baseline (s1\_budget, 43.0\%) by
14.3\% at lower cost (2.56$\times$ vs.\ 3.89$\times$).

\paragraph{Llama-3.1-8B.} Direction reversal interacts with model
capacity. On TWExpress, CaTS collapses to near-base performance
(35.8\% vs.\ base 36.5\%) because it rarely triggers, while DIAL
achieves 94.8\%. On FEVER, DIAL (54.7\%) is competitive with CaTS
(56.3\%), unlike Qwen3 where FEVER is the weakest environment.
On APPS, AUQ drops to
46.0\% (below the no-trigger baseline), demonstrating that wrong-direction
gating actively harms performance on this backbone.

% ============================================================
% F. SIGNAL AND FEATURE ANALYSES
% ============================================================
\section{Signal and Feature Analyses}
\label{app:signal-analyses}

This appendix collects empirical analyses of signal identity
and signal sufficiency referenced in
Section~\ref{sec:empirical-landscape}.
\S\ref{app:feature-selection} reports which features the
$\ell_1$ gate selects per environment and how their signed weights
align with the Two-Source Model.
\S\ref{app:llm-features} ablates the LLM feature layer.
\S\ref{app:regularizer-ablation} stress-tests sparsity by
ablating the regularizer under an expanded candidate pool.
\S\ref{app:auc} reports the AUC hierarchy
(single-signal $<$ multi-signal $<$ hidden-state probe) as the
empirical signature of $\sigma$-insufficiency.

\subsection{Per-Environment Feature Selection}
\label{app:feature-selection}

The LASSO coefficient matrix identifies which auxiliary features
carry information about the latent state type $\tau$ in each
environment. These are exactly the components of $\phi(s)$ that
lift $\sigma$ toward sufficiency for VOC (App~\ref{app:proof-necessity}).
\texttt{step\_count} is informative in 4 of 5 environments with
significant signal: it proxies \emph{information accumulation},
distinguishing late-step Type~I states (where information has
accumulated to a known shortage) from early-step Type~D states
(where multiple options are still viable). WebShop is the
exception, relying on \texttt{num\_available\_actions} and
task-specific features (\texttt{price\_mentioned},
\texttt{action\_is\_click}); these proxy \emph{decision optionality}
directly, as expected for a choice-dominated (Type~D) environment.
In every case, the features LASSO retains are interpretable as
proxies for the two-source decomposition rather than ad-hoc
engineering choices.

\subsection{LLM-Generated Features}
\label{app:llm-features}

\paragraph{LLM feature ablation.}
\begin{wraptable}{r}{0.45\textwidth}
\vspace{-12pt}
\caption{LLM feature ablation. SR (\%) with and without
LLM-generated features.}
\label{tab:llm-ablation}
\centering
\resizebox{0.43\textwidth}{!}{%
\begin{tabular}{lccc}
\toprule
Environment & w/ LLM & w/o LLM & $\Delta$ \\
\midrule
HotpotQA  & 95.2 & 93.8 & $+$1.4 \\
WebShop   & 43.8 & 42.7 & $+$1.1 \\
FEVER     & 49.8 & 40.7 & $+$9.1 \\
TWExpress & 99.0 & 97.2 & $+$1.8 \\
Plancraft & 23.3 & 23.4 & $-$0.2 \\
APPS      & 73.0 & 73.0 & $+$0.0 \\
\bottomrule
\end{tabular}%
}
\vspace{-10pt}
\end{wraptable}
Table~\ref{tab:llm-ablation} compares DIAL with and without
LLM-generated features. FEVER is the only environment where
LLM features contribute meaningfully ($+$9.1\%), where
task-specific signals (\texttt{has\_claim}, \texttt{text\_length})
extend the universal feature pool. In the remaining environments,
removing LLM features changes SR by ${<}$2\%; in Plancraft and
APPS the contribution is effectively zero. The LLM feature
layer is therefore an automation convenience rather than a
core driver of DIAL's gains.

\paragraph{Selected LLM features per environment.}
Table~\ref{tab:llm-features} reports the features generated by the
LLM reasoning step (Phase~2) for each environment. The LLM receives
exploration statistics and representative positive/negative examples,
then produces a Python function extracting 5 task-specific features
(the full prompt template is in \S\ref{app:llm-prompts}). LASSO
subsequently selects or discards each feature based on its predictive
value for rollout utility.

\begin{table*}[h!]
\centering
\caption{LLM-generated features per environment. Each environment
receives 5 LLM-proposed features; LASSO selects those with non-zero
coefficients. Environments with richer task-specific structure
(WebShop, TWExpress) retain more LLM features.}
\label{tab:llm-features}
\resizebox{\textwidth}{!}{%
\begin{tabular}{lcl}
\toprule
\textbf{Environment} & \textbf{Sel./Gen.}
  & \textbf{LASSO-Selected LLM Features} \\
\midrule
WebShop    & 4/5 & \texttt{price\_mentioned},
  \texttt{action\_is\_click}, \texttt{step\_early},
  \texttt{instruction\_keyword\_count} \\
TWExpress  & 4/5 & \texttt{text\_length},
  \texttt{closed\_ratio}, \texttt{action\_look\_around},
  \texttt{already\_open} \\
HotpotQA   & 3/5 & \texttt{has\_question},
  \texttt{entity\_count}, \texttt{fact\_keyword\_count} \\
FEVER      & 2/5 & \texttt{text\_length\_normalized},
  \texttt{has\_claim} \\
APPS       & 1/5 & \texttt{action\_type} \\
Plancraft  & 0/5 & (none selected) \\
\bottomrule
\end{tabular}%
}
\end{table*}

Environments with rich task-specific structure benefit most from LLM
features: WebShop's choice-dominated interface is well captured by
\texttt{price\_mentioned} and \texttt{action\_is\_click}, while
TWExpress's exploration mechanics are reflected in
\texttt{closed\_ratio} and \texttt{action\_look\_around}. In
contrast, Plancraft's rollout-harmful nature means no LLM feature
improves prediction, and LASSO correctly filters all five proposals.
This pattern aligns with the LLM ablation results
(Table~\ref{tab:llm-ablation}): the environment where LLM features
contribute most to SR (FEVER, $+$9.1\%) is precisely where the
LLM discovers unique signals (\texttt{has\_claim}) that the
universal feature pool lacks.

Figure~\ref{fig:llm-output-webshop} shows the complete LLM-generated feature extraction function for WebShop (seed~42), the environment with the most LLM features retained.

\begin{figure}[h]
\begin{tcolorbox}[
    title=\textbf{LLM Output for WebShop},
    colframe=PTGreenDark,
    colback=PTLightGreen,
    colbacktitle=PTGreenDark,
    coltitle=white,
    fonttitle=\bfseries,
    boxrule=0.9pt,
    arc=0pt,
    left=7pt,right=7pt,top=4pt,bottom=4pt,
    width=\columnwidth
  ]
\begin{lstlisting}[language=Python,basicstyle=\ttfamily\scriptsize,
  breaklines=true,columns=fullflexible,
  keywordstyle=\color{blue},commentstyle=\color{gray},
  showstringspaces=false]
def extract_features(state_text, step_count, action_text):
    import re
    features = {}
    # Instruction keyword density
    keywords = ["gluten free", "protein",
        "price lower than", "size", "color"]
    matches = sum(1 for kw in keywords
        if kw in state_text.lower())
    features['instruction_keyword_count'] = (
        float(matches) / len(keywords))
    # Early-step indicator
    features['step_early'] = (
        1.0 if step_count <= 2 else 0.0)
    # Click vs. search action
    features['action_is_click'] = (
        1.0 if action_text.startswith("click[")
        else 0.0)
    # Price constraint presence
    price_matches = len(re.findall(
        r"price lower than|price:|\$",
        state_text, re.IGNORECASE))
    features['price_mentioned'] = (
        float(price_matches) / 5.0)
    # Filter complexity
    filter_kws = ["size","color","flavor",
        "price","rating"]
    features['filter_count'] = sum(
        1 for kw in filter_kws
        if kw in state_text.lower()) / 5.0
    return features
\end{lstlisting}
\end{tcolorbox}
\caption{Complete LLM-generated feature extraction function for WebShop (seed 42), the environment with the most LLM features retained by LASSO.}
\label{fig:llm-output-webshop}
\end{figure}

\subsection{Regularizer Ablation}
\label{app:regularizer-ablation}

The signed-weight diagnostic of \S\ref{sec:learn} and DIAL's
robustness to noisy LLM-proposed features (\S\ref{sec:reason})
both rely on the \emph{exact} sparsity induced by the $\ell_1$
penalty in Eq.~\eqref{eq:lasso-full}. To verify that $\ell_1$
specifically, rather than any shrinkage or selection mechanism,
is responsible for these properties, we ablate the regularizer
while holding all other pipeline components fixed: same
exploration dataset $\mathcal{D}$, same candidate feature pool
$\phi_{\text{cand}}$, same 5-fold CV protocol for hyperparameters,
same threshold $\tau{=}0.5$ (fixed across variants for this
ablation to isolate the regularizer's effect; main DIAL uses
$\tau$ selected by cross-validation per Eq.~\eqref{eq:gate}). We compare five variants:

\begin{itemize}
\item \textbf{$\ell_1$-logistic} (DIAL default): joint sparsity and
shrinkage.
\item \textbf{$\ell_2$-logistic} (Ridge): shrinkage only; all
features retained with small non-zero weights.
\item \textbf{No regularization}: plain logistic fit on
$\phi_{\text{cand}}$.
\item \textbf{Elastic Net} ($\alpha{=}0.5$, mixing $\ell_1$ and
$\ell_2$).
\item \textbf{MI top-$k$ + unregularized}: hard feature selection
by mutual information (top-3), then unregularized logistic on the
retained features. This variant isolates \emph{selection} from
\emph{shrinkage}.
\end{itemize}

Table~\ref{tab:regularizer-ablation} reports SR and the mean number
of non-zero fitted weights across three environments spanning
signal strength: HotpotQA ($|\rho^*|{=}0.494$, strong),
WebShop ($|\rho^*|{=}0.444$, moderate), and Plancraft
($|\rho^*|{=}0.016$, weak).

\begin{table}[h]
\caption{Regularizer ablation on three environments ordered by
signal strength. SR (\%); ``\#nz'' is the mean number of non-zero
fitted weights across three seeds. For this ablation, $\phi_{\text{cand}}$ is expanded to a
stress-test pool of 31 features ($5$~universal $+\,3$~derived
$+\,20$~hidden-state PCA $+\,3$~auto-extracted) to evaluate each
regularizer's behavior under a high-dimensional candidate set;
adding the LLM feature layer brings the pool to $\approx$36.
$\ell_1$ is the only variant that achieves competitive SR
\emph{and} exact sparsity required by the signed-weight
diagnostic.}
\label{tab:regularizer-ablation}
\centering\small
\begin{tabular}{lcccccc}
\toprule
 & \multicolumn{2}{c}{HotpotQA} & \multicolumn{2}{c}{WebShop} & \multicolumn{2}{c}{Plancraft} \\
 & \multicolumn{2}{c}{$|\rho^*|{=}0.494$} & \multicolumn{2}{c}{$|\rho^*|{=}0.444$} & \multicolumn{2}{c}{$|\rho^*|{=}0.016$} \\
\cmidrule(lr){2-3}\cmidrule(lr){4-5}\cmidrule(lr){6-7}
Regularizer & SR & \#nz & SR & \#nz & SR & \#nz \\
\midrule
$\ell_1$ (DIAL)       & $94.8\pm0.2$ & 6.7 & $42.7\pm6.7$ & 6.3 & $27.2\pm1.6$ & 5.0 \\
$\ell_2$              & $94.3\pm1.0$ & 30  & $41.3\pm5.8$ & 30  & $26.3\pm1.0$ & 30  \\
No regularization     & $94.0\pm1.9$ & 30  & $36.2\pm7.3$ & 30  & $27.8\pm0.9$ & 30  \\
Elastic Net           & $94.3\pm1.6$ & 8.7 & $44.0\pm1.8$ & 6.0 & $27.2\pm0.8$ & 10  \\
MI top-3 + unreg.     & $95.0\pm1.1$ & 3.0 & $38.5\pm8.2$ & 3.0 & $26.3\pm1.2$ & 3.0 \\
\bottomrule
\end{tabular}
\end{table}

\noindent\textit{Note.} SR is reported as mean$\pm$std across three
seeds (Elastic Net on WebShop is reported for a single seed and
is excluded from std-based comparisons).
``\#nz'' is the mean number of features with non-zero coefficient;
for $\ell_2$, no~regularization, and MI~top-$k$ this number is
fixed by construction. The L1 row uses the same fitting pipeline
as the four other rows (the modified \texttt{principled\_scg.py}
with the LASSO step swapped per row), yielding internally
consistent SR across the table; the slight numerical
divergence from the corresponding cells in the main results
(Table~\ref{tab:full-results}) on Plancraft reflects an
implementation difference there
(\texttt{scg\_finetune\_lr}, default sklearn logistic) rather than
a methodological change. Across $\ell_1$ runs, sparsity adapts to
signal strength: HotpotQA and WebShop retain $6$--$7$ features
each, whereas Plancraft collapses to $\approx$5, absorbing the
remainder into the bias and yielding a near-constant trigger
probability consistent with that environment's low rollout
headroom (Figure~\ref{fig:trigger-adapt}).

The five rows isolate three hypotheses underlying DIAL's
regularizer choice.

\paragraph{$\ell_1$ vs.\ $\ell_2$: sparsity is required for
diagnosis, not just for predictive performance.}
Both penalties shrink noisy coefficients toward zero, and on
all three environments their gating SRs are within a single
percentage point: $94.8$~vs.\ $94.3$ on HotpotQA,
$42.7$~vs.\ $41.3$ on WebShop, $27.2$~vs.\ $26.3$ on Plancraft.
The critical distinction is interpretability rather than
performance: under $\ell_2$, every one of the $\approx$30
candidate features retains a small non-zero weight, collapsing
the clean \{informative, uninformative\} partition of
\S\ref{sec:learn} into a continuous ranking. The signed-weight
summary no longer reads off which signals the environment treats
as informative; it merely ranks them. Reports of ``non-zero
weights per environment'' in the main text (e.g., the direction
analyses of Table~\ref{tab:wrong-direction}) are only
interpretable under the exact sparsity that $\ell_1$ provides.

\paragraph{$\ell_1$ vs.\ unregularized logistic: robustness under
weak signals.}
The unregularized fit is the most fragile variant in the table.
On Plancraft ($|\rho^*|{\approx}0.02$), no~regularization actually
overfits the training split into a higher mean SR than $\ell_1$
($27.8$ vs.\ $27.2$) but with no diagnostic interpretation
($\#$nz~$=30$ across all seeds). The damaging case is WebShop
($|\rho^*|{\approx}0.44$, moderate signal): no~regularization
drops to $36.2\pm7.3$, fully $6.5$ points below $\ell_1$ at
$42.7\pm6.7$, with the wider standard deviation reflecting
seed-dependent overfitting onto whichever PCA components happen
to pseudo-correlate with utility on the training fold.
$\ell_1$'s shrinkage absorbs this noise into the bias term and
yields a gate that triggers more conservatively (lower
\texttt{rollouts/ep} at the same SR; cf.\
Appendix~\ref{app:trigger-rate}).

\paragraph{$\ell_1$ vs.\ MI top-$k$ + unregularized: joint
selection+shrinkage vs.\ selection alone.}
Hard feature selection by mutual information separates the two
roles $\ell_1$ plays. MI~top-3 recovers most of $\ell_1$'s SR
on strong- and weak-signal environments ($95.0$~vs.\ $94.8$ on
HotpotQA, $26.3$~vs.\ $27.2$ on Plancraft), confirming that on
those endpoints sparsity matters more than which exact features
are kept. On WebShop the gap opens to $4$ points ($38.5$ vs.\
$42.7$, std~$=8.2$, the largest in the table), reflecting
seed-dependent overfitting that $\ell_1$'s joint shrinkage
prevents. MI~top-$k$ also requires tuning $k$ per environment
rather than CVing a single $C$, reintroducing the
per-environment engineering cost DIAL aims to eliminate.

\paragraph{$\ell_1$ vs.\ Elastic Net.}
Elastic Net ($\alpha{=}0.5$) is statistically interchangeable
with $\ell_1$ on all three environments (within $1$~pt SR;
$\#$nz higher by $1$--$5$). We default to $\ell_1$ as the
strictly simpler choice: one hyperparameter instead of two,
exact zeros instead of small-but-nonzero weights, and direct
compatibility with the signed-weight diagnostic.

\paragraph{Summary.}
$\ell_1$ is the only variant in this table that simultaneously
delivers competitive SR, the exact sparsity required by the
signed-weight diagnostic of \S\ref{sec:learn}, and a single
CV-tunable hyperparameter; we use it as the default.

\subsection{AUC Hierarchy: Empirical Signature of $\sigma$-Insufficiency}
\label{app:auc}

Table~\ref{tab:auc-hierarchy} is the empirical instantiation of
the sufficient-statistic claim of \S\ref{sec:toy-model}: it
measures, on held-out data, how much of $\mathrm{VOC}$ each gate
family can recover. Single entropy is near chance
(AUC${\approx}$0.50) in every environment, the operational
signature of $\sigma$'s insufficiency for VOC. Multi-signal
logistic gates that augment $\sigma$ with universal auxiliary
features reach 0.74--0.92, recovering most of the available
signal; hidden-state probes (which have access to the model's
full internal representation, an upper bound on what any feature
map $\phi(s)$ can extract) reach 0.79--0.99. The gap between the
single-entropy row and the multi-signal row is therefore a
quantitative measure of how much auxiliary conditioning is
required to lift $\sigma$ toward sufficiency for VOC; the
residual gap between multi-signal and hidden-state probe is the
sufficiency that universal features fail to capture (small in
most environments). The bottleneck is signal \emph{combination},
not gate \emph{capacity}, and the multi-signal LR row is the
empirical case for DIAL's design.

\begin{table}[H]
\caption{AUC by gate family. Single entropy is near chance;
multi-signal gates close most of the gap to hidden-state probes.
FEVER and TWExpress are omitted because their short episodes
yield insufficient held-out data for reliable AUC estimation.}
\label{tab:auc-hierarchy}
\centering\small
\begin{tabular}{lcccc}
\toprule
Gate family & HotpotQA & APPS & WebShop & Plancraft \\
\midrule
Single entropy      & 0.50 & 0.56 & 0.50 & 0.50 \\
Best single signal  & 0.78 & 0.78 & 0.90 & 0.74 \\
Multi-signal LR     & 0.85 & 0.76 & 0.92 & 0.74 \\
Hidden-state probe  & 0.87 & 0.79 & 0.99 & 0.95 \\
\bottomrule
\end{tabular}
\end{table}

% ============================================================
% G. ROBUSTNESS AND ADDITIONAL EXPERIMENTS
% ============================================================
\section{Robustness and Additional Experiments}
\label{app:robustness}

This appendix collects Two-Source Model verification data,
robustness checks for direction reversal, and additional
experimental analyses.
\S\ref{app:reversal-norm} and \S\ref{app:reversal-alternatives}
rule out scale and reward-distribution artifacts as causes of the
observed sign reversal.
\S\ref{app:temporal}, \S\ref{app:cross-setting}, and
\S\ref{app:multi-backbone} verify the Two-Source Model's
predictions across time, environment composition, and model
backbone.
\S\ref{app:trigger-rate} analyzes DIAL's deployment trigger
behavior across environments.

\subsection{Robustness of Reversal to Entropy Quantile-Normalization}
\label{app:reversal-norm}

A natural concern with Table~\ref{tab:signal-discovery} is that
cross-backbone reversal could reflect incomparable raw entropy
scales across tokenizers rather than a structural sign change.
We rule this out by replacing each $\sigma(s)$ with its empirical
quantile rank $Q(s) \!\in\! [0, 1]$ under three pooling schemes:
\textbf{S1} ranks within each (environment, backbone) cell,
\textbf{S2} ranks within each backbone's pool of states (across
environments), and \textbf{S3} ranks within each environment's
pool (across backbones). For each scheme we recompute Spearman
$\rho(Q, U)$ and Pearson $\rho(Q, U)$ per cell with 1000-resample
bootstrap confidence intervals.

\paragraph{Spearman: invariant by construction.} Per-cell
Spearman $\rho$ is mathematically rank-invariant under any
within-cell monotone transform of $\sigma$, so all three schemes
give identical Spearman $\rho$ to the raw column within each
cell. We report this transparently: it confirms that rank
structure is preserved but is not, in itself, a non-trivial test.

\paragraph{Pearson: scale-sensitive, the substantive test.}
Pearson $\rho$ varies under non-linear normalization. We
partition the 18 (env, backbone) cells by raw Pearson $|\rho|$
at the 0.13 threshold and report the 11 strong-signal cells
($|\rho|\geq 0.13$) in Table~\ref{tab:phase-a-norm}.
\textbf{All 11 strong-signal cells are robust}: Pearson sign is
stable under all three normalization schemes on 10 cells and
2/3 schemes on 1 cell (Phi-3.5 APPS, where S3 yields $-0.054$
with CI straddling zero). Critically, every cell that the
paper's reversal claim invokes, including the cross-backbone
anchors on HotpotQA, APPS, and FEVER, is in this strong block.
The remaining 7 cells fall into two categories that we exclude
from the table for brevity. \textbf{Weak-signal cells} (4 of 7:
Qwen3 WebShop, Phi-3.5 HotpotQA, and Plancraft on \{Qwen3,
Llama-3.1\}) all have raw $|\rho|<0.13$ and CIs already
straddling zero; under each of S1/S2/S3 they remain
inconclusive, but since the paper does not claim reversal on
these cells, this is expected sample-level noise rather than a
refutation. \textbf{Insufficient-data cells} (3 of 7: TWExpress
on \{Phi-3.5, Llama-3.1\} and Plancraft on Phi-3.5) have logged
$\sigma$ with $\leq 2$ unique values or no positive $U$ labels,
making correlation analysis statistically ill-defined. The
pattern in the strong block, where stability tracks signal
strength and no strong-signal cell flips under any scheme, is
what one expects if the reversal is structural rather than a
tokenizer-scale artifact.

\begin{table*}[t]
\centering
\caption{Robustness of cross-backbone reversal to entropy
quantile-normalization. We report Spearman $\rho(\sigma,U)$
(invariant under within-cell monotone transforms, so shared
across raw/S1/S2/S3) and Pearson $\rho$ under raw and three
quantile schemes: \textbf{S1} ranks within each (env, backbone),
\textbf{S2} within each backbone's pool, \textbf{S3} within each
environment's pool. Listed are the \textbf{11 strong-signal
cells} (raw Pearson $|\rho|\geq 0.13$) covering every reversal
claim. \textbf{Robust} = sign stable vs.\ raw under all three
schemes (CIs on the same side).
$^{\ddagger}$ Phi-3.5 APPS robust under S1, S2 but inconclusive
under S3. Excluded (discussed in prose): 4 weak-signal cells
($|\rho|<0.13$, raw CIs straddle zero) and 3 cells with
insufficient logged data. $^*$: $p<0.05$, 1000-resample
bootstrap.}
\label{tab:phase-a-norm}
\footnotesize
\begin{tabular}{l l r c c c c c l}
\toprule
& & & \textbf{Spearman} & \multicolumn{4}{c}{\textbf{Pearson $\rho(\sigma, U)$}} & \\
\cmidrule(lr){5-8}
\textbf{Backbone} & \textbf{Env} & \textbf{N} & raw=S1=S2=S3 & raw & S1 & S2 & S3 & \textbf{Conclusion} \\
\midrule
Llama-3.1 & FEVER     & 840  & $+0.428^*$ & $+0.481$ & $+0.418$ & $+0.454$ & $+0.432$ & \cmark\ Robust \\
Phi-3.5   & WebShop   & 751  & $+0.335^*$ & $+0.384$ & $+0.374$ & $+0.382$ & $+0.394$ & \cmark\ Robust \\
Llama-3.1 & HotpotQA  & 244  & $-0.346^*$ & $-0.344$ & $-0.337$ & $-0.341$ & $-0.320$ & \cmark\ Robust \\
Qwen3-4B  & TWExpress & 798  & $-0.290^*$ & $-0.319$ & $-0.290$ & $-0.329$ & $-0.324$ & \cmark\ Robust \\
Llama-3.1 & WebShop   & 948  & $+0.287^*$ & $+0.298$ & $+0.287$ & $+0.303$ & $+0.321$ & \cmark\ Robust \\
Llama-3.1 & APPS      & 475  & $-0.242^*$ & $-0.217$ & $-0.271$ & $-0.233$ & $-0.287$ & \cmark\ Robust \\
Qwen3-4B  & APPS      & 439  & $+0.317^*$ & $+0.203$ & $+0.303$ & $+0.265$ & $+0.285$ & \cmark\ Robust \\
Phi-3.5   & APPS      & 400  & $-0.129^*$ & $-0.197$ & $-0.104$ & $-0.193$ & $-0.054$ & \cmark\ Robust \\
Qwen3-4B  & FEVER     & 282  & $-0.119^*$ & $-0.186$ & $-0.123$ & $-0.237$ & $-0.192$ & \cmark\ Robust \\
Phi-3.5   & FEVER     & 824  & $-0.156^*$ & $-0.158$ & $-0.161$ & $-0.160$ & $-0.157$ & \cmark\ Robust \\
Qwen3-4B  & HotpotQA  & 1208 & $-0.327^*$ & $-0.138$ & $-0.330$ & $-0.264$ & $-0.323$ & \cmark\ Robust \\
\bottomrule
\end{tabular}
\end{table*}

The cross-backbone reversal pattern on HotpotQA, APPS, and FEVER
therefore survives all three normalization schemes, supporting
the claim of \S\ref{sec:empirical-landscape} that the reversal
is structural rather than a tokenizer-scale artifact.

\subsection{Ruling Out Alternative Explanations: Reward and
            Entropy Calibration}
\label{app:reversal-alternatives}

We further verify that the reversal is not explained by reward
parameterization or entropy estimator calibration.

\paragraph{Reward transformations.} Spearman $\rho(\sigma, U)$
is rank-invariant under positive monotone transforms of $U$,
so $U \!\to\! \alpha U$ for $\alpha > 0$ exactly preserves
$\rho$ across all 15 analyzable cells; for $\alpha < 0$ the
sign flips as expected (rank order reversed).
Table~\ref{tab:phase-a-alternatives} shows the $\alpha\!=\!-1$
column for completeness. Reward parameterization is therefore
not a confound. Additive shifts $U \!\to\! U + c$ also preserve
the correlation, but trivially flip the binary classifier label
$\mathrm{sign}(U)$ on states with $U\!=\!0$; this is a
degenerate property of the $U$ distribution rather than a
sensitivity, and we do not include label-flip data in the
robustness verdict.

\paragraph{Entropy transformations.} For $\sigma \!\to\!
\sigma/T$ with $T \in \{0.5, 1, 2, 5\}$, both Spearman and
Pearson $\rho$ are exactly preserved (positive linear scaling).
Under non-linear transforms $\sigma \!\to\! \sigma^\alpha$ with
$\alpha \in \{0.5, 1, 2\}$ and $\sigma \!\to\! \log(\sigma +
\epsilon)$, Spearman $\rho$ remains invariant by rank-preservation
on every cell; Pearson varies (as expected for non-linear
transforms) but only one cell, Phi-3.5 / APPS under
$\log\sigma$, exhibits a Pearson sign flip
(Table~\ref{tab:phase-a-alternatives}). Spearman on that cell
is stable at $-0.129$, so the structural reversal claim is
unaffected. We report this single anomaly transparently rather
than averaging it away.

\paragraph{Trajectory length.}
Trajectory length varies across environments and correlates with
both entropy and utility. To rule out length as a confounder, we
stratify by step count and recompute $\rho(\text{entropy}, U)$
within each stratum. Direction reversal persists within
trajectory-length strata where utility variance is non-zero. In
HotpotQA, $\rho$ remains negative across all three strata
($-0.18$/$-0.46$/$-0.42$). The reversal is not an artifact of
trajectory length.

\begin{table*}[t]
\centering
\caption{Robustness of $\rho(\sigma, U)$ to monotone transforms.
Spearman $\rho$ is rank-invariant under monotone transforms of
either variable, so the $\sigma^{0.5}$, $\sigma^{2}$, and
$\log\sigma$ columns must match raw up to numerical noise (a
sanity check on the entropy pipeline); under
$U\!\to\!\alpha U$ with $\alpha<0$, sign flips as expected.
Pearson $\rho$ varies under non-linear $\sigma$ transforms; the
only Pearson sign flip is Phi-3.5/APPS under $\log\sigma$
($+0.027$), where Spearman stays stable at $-0.129$. Excluded:
three cells with $|\sigma_{\mathrm{unique}}|\!\leq\!2$ or no
positive $U$ labels (TWExpress on \{Phi-3.5, Llama-3.1\},
Plancraft on Phi-3.5), where correlation is ill-defined.}
\label{tab:phase-a-alternatives}
\footnotesize
\resizebox{\textwidth}{!}{%
\begin{tabular}{l l c c c c c c c c}
\toprule
& & \multicolumn{5}{c}{\textbf{Spearman $\rho$}} & \multicolumn{3}{c}{\textbf{Pearson $\rho$}} \\
\cmidrule(lr){3-7} \cmidrule(lr){8-10}
\textbf{Backbone} & \textbf{Env} & raw & $\alpha U$, $\alpha\!=\!-1$ & $\sigma^{0.5}$ & $\sigma^{2}$ & $\log\sigma$ & $\sigma^{0.5}$ & $\sigma^{2}$ & $\log\sigma$ \\
\midrule
Qwen3-4B & HotpotQA  & $-0.327$ & $+0.327$ & $-0.327$ & $-0.327$ & $-0.327$ & $-0.201$ & $-0.085$ & $-0.317$ \\
Qwen3-4B & APPS      & $+0.317$ & $-0.317$ & $+0.317$ & $+0.317$ & $+0.317$ & $+0.273$ & $+0.085$ & $+0.305$ \\
Qwen3-4B & WebShop   & $+0.133$ & $-0.133$ & $+0.133$ & $+0.133$ & $+0.133$ & $+0.127$ & $+0.008$ & $+0.238$ \\
Qwen3-4B & FEVER     & $-0.119$ & $+0.119$ & $-0.119$ & $-0.119$ & $-0.119$ & $-0.228$ & $-0.140$ & $-0.240$ \\
Qwen3-4B & TWExpress & $-0.290$ & $+0.290$ & $-0.290$ & $-0.290$ & $-0.290$ & $-0.305$ & $-0.281$ & $-0.264$ \\
Qwen3-4B & Plancraft & $-0.016$ & $+0.016$ & $-0.016$ & $-0.016$ & $-0.016$ & $-0.019$ & $-0.031$ & $-0.010$ \\
\midrule
Phi-3.5 & HotpotQA & $+0.184$ & $-0.184$ & $+0.184$ & $+0.184$ & $+0.184$ & $+0.170$ & $+0.042$ & $+0.107$ \\
Phi-3.5 & APPS     & $-0.129$ & $+0.129$ & $-0.129$ & $-0.129$ & $-0.129$ & $-0.161$ & $-0.196$ & $+0.027$ \\
Phi-3.5 & WebShop  & $+0.335$ & $-0.335$ & $+0.335$ & $+0.335$ & $+0.335$ & $+0.399$ & $+0.324$ & $+0.296$ \\
Phi-3.5 & FEVER    & $-0.156$ & $+0.156$ & $-0.156$ & $-0.156$ & $-0.156$ & $-0.153$ & $-0.156$ & $-0.049$ \\
\midrule
Llama-3.1 & HotpotQA  & $-0.346$ & $+0.346$ & $-0.346$ & $-0.346$ & $-0.346$ & $-0.332$ & $-0.335$ & $-0.288$ \\
Llama-3.1 & APPS      & $-0.242$ & $+0.242$ & $-0.242$ & $-0.242$ & $-0.242$ & $-0.266$ & $-0.150$ & $-0.326$ \\
Llama-3.1 & WebShop   & $+0.287$ & $-0.287$ & $+0.287$ & $+0.287$ & $+0.287$ & $+0.310$ & $+0.258$ & $+0.310$ \\
Llama-3.1 & FEVER     & $+0.428$ & $-0.428$ & $+0.428$ & $+0.428$ & $+0.428$ & $+0.431$ & $+0.528$ & $+0.354$ \\
Llama-3.1 & Plancraft & $-0.041$ & $+0.041$ & $-0.041$ & $-0.041$ & $-0.041$ & $-0.014$ & $-0.035$ & $-0.003$ \\
\bottomrule
\end{tabular}%
}
\end{table*}

Across the four alternative explanations considered (rollout
noise, reward bias, search depth, calibration error), only
rollout noise and search depth would require new rollouts to
test directly. Reward bias and calibration error are post-hoc
verifiable on the logged data and are ruled out above. We leave
the rollout-noise and search-depth ablations as future work; the
two-source mixture interpretation of \S\ref{sec:toy-model} is
consistent with all observational data and survives all
post-hoc robustness checks within reach.

\subsection{Temporal Dynamics (P1)}
\label{app:temporal}

Table~\ref{tab:temporal} reports the full early/late $\rho$ data
referenced in Figure~\ref{fig:temporal-shift}. The Two-Source Model
predicts that $\rho(\text{entropy}, U)$ should decrease from early
to late steps within each episode, as late-step uncertainty
increasingly reflects the residual Type~I component after
information-gathering attempts.

\begin{table}[h]
\centering
\caption{Temporal shift in $\rho$(entropy, $U$). All environments
show $\rho$ shifting more negative from early to late steps,
consistent with P1. Plancraft omitted ($|\rho^*|{=}0.016$, too weak
for meaningful comparison).}
\label{tab:temporal}
\begin{tabular}{llccccl}
\toprule
\textbf{Env} & \textbf{Phase} & $\rho$ & \textbf{95\% CI}
  & $n$ & $p$ & \textbf{Shift} \\
\midrule
HotpotQA & early & $-$0.176 & [$-$0.260, $-$0.091] & 509
  & $<$0.001
  & \multirow{2}{*}{$\Delta{=}{-}0.261$} \\
  & late  & $-$0.437 & [$-$0.500, $-$0.370] & 699
  & $<$0.001 & \\
\midrule
APPS & early & $+$0.102 & [$-$0.013, $+$0.209] & 290
  & 0.084
  & \multirow{2}{*}{$\Delta{=}{-}0.246$} \\
  & late  & $-$0.144 & [$-$0.281, $+$0.001] & 199
  & 0.043 & \\
\midrule
WebShop & early & $+$0.285 & [$+$0.222, $+$0.346] & 600
  & $<$0.001
  & \multirow{2}{*}{$\Delta{=}{-}0.291$} \\
  & late  & $-$0.006 & n.s. & 473 & 0.895 & \\
\midrule
FEVER & early & $+$0.054 & n.s. & 200 & 0.446
  & \multirow{2}{*}{$\Delta{=}{+}0.024$} \\
  & late  & $+$0.078 & n.s. & 82 & 0.486 & \\
\midrule
TWExpress & early & $+$0.161 & --- & 422 & $<$0.001
  & \multirow{2}{*}{$\Delta{=}{-}0.153$} \\
  & late  & $+$0.008 & n.s. & 376 & 0.877 & \\
\bottomrule
\end{tabular}
\end{table}

The prediction holds across 4 of 5 testable environments:
HotpotQA shows the strongest shift ($\Delta{=}{-}0.261$), with
late-step entropy reflecting deeper information poverty after failed
evidence retrieval. WebShop's strong positive early-phase signal
($\rho{=}{+}0.285$) vanishes in the late phase ($\rho{=}{-}0.006$,
n.s.), consistent with the choice-dominated structure being resolved
early. In APPS, $\rho$ reverses from marginally positive ($+$0.102)
to negative ($-$0.144), as late-step decision complexity gives way
to residual information poverty.

FEVER is the exception: both phases show weak, non-significant
positive $\rho$, consistent with FEVER's short episodes
(median 8 steps) concentrating most rollout value at step~0--1,
leaving insufficient late-step variation for a meaningful shift.

\subsection{Cross-setting Consistency at Fixed Backbone}
\label{app:cross-setting}

Beyond the within-episode (P1) and within-task interventional
(P2) tests in the main text, the Two-Source Model implies that
settings with similar $p_I$ should exhibit similar $\rho$. With
the backbone held fixed at Qwen3
(Table~\ref{tab:signal-discovery}), this consistency holds:
FEVER and HotpotQA (both search-based QA) share negative $\rho$,
and WebShop and APPS (both involving choice among alternatives)
share positive $\rho$. As predicted by the (environment,
backbone)-indexed mixture of \S\ref{sec:toy-model}, however,
task-structure similarity alone does not pin down $\rho$ when
the backbone changes: the same environment can sit on opposite
sides of $p_I^*$ under different LLMs
(Appendix~\ref{app:multi-backbone}). We therefore treat this
cross-setting consistency as a complementary observational
check rather than a standalone falsifiable prediction; the
causal evidence in P2 is the more decisive test of the
mixture-driven account.

\begin{figure*}[t]
	\centering
	\includegraphics[width=\textwidth]{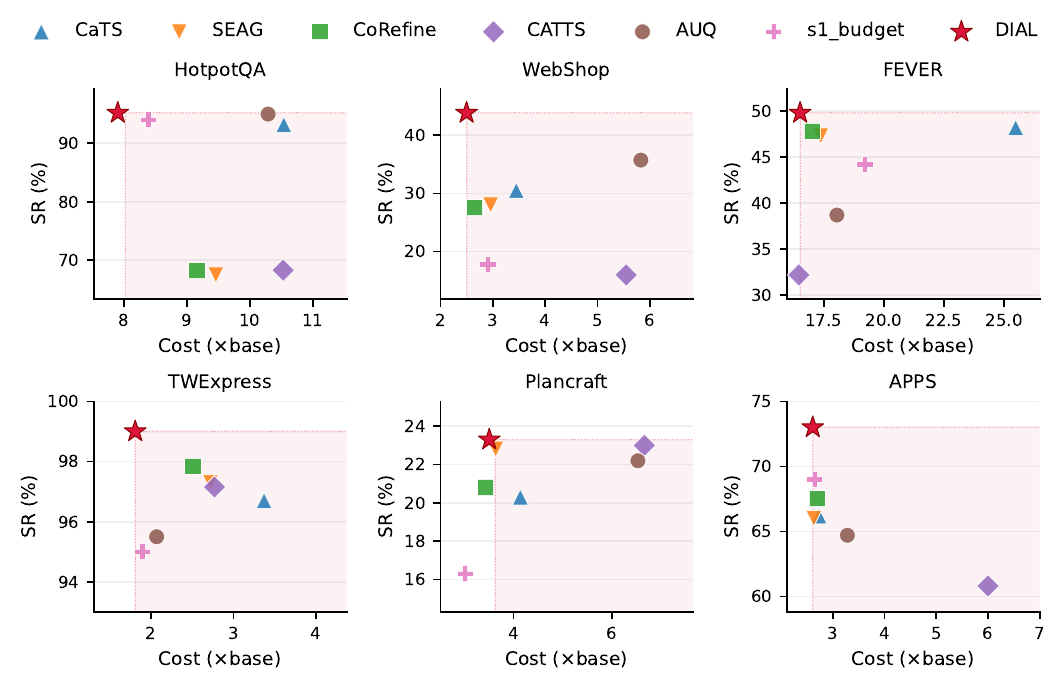}
	\caption{SR vs.\ Cost ($\times$base) across all 6 environments
		on Qwen3-4B. Shaded region: dominated by DIAL (lower SR
		\emph{and} higher cost). Full numerical results across three
		backbones are in Table~\ref{tab:full-results}.}
	\label{fig:pareto}
\end{figure*}
\subsection{Multi-Backbone Verification}
\label{app:multi-backbone}

We validate our core findings on three backbones from three vendors:
Qwen3-4B-Instruct (Alibaba, 4B), Phi-3.5-mini-instruct (Microsoft,
3.8B), and Llama-3.1-8B-Instruct (Meta, 8B) across all 6
environments. Table~\ref{tab:full-results}
(\S\ref{sec:main-results}) reports full numerical results, and
Figure~\ref{fig:pareto} below visualizes the Qwen3-4B Pareto
frontier.

This appendix supports the (environment, backbone)-indexed
formulation of \S\ref{sec:toy-model}: the same task changes
$p_I$ (and, when it crosses $p_I^*$, the sign of $\rho$) when
the backbone changes, so backbone-level reversal is a prediction
of the two-source model rather than a confound to be controlled
for.

\begin{figure*}[h!]
	\centering
	\includegraphics[width=\textwidth]{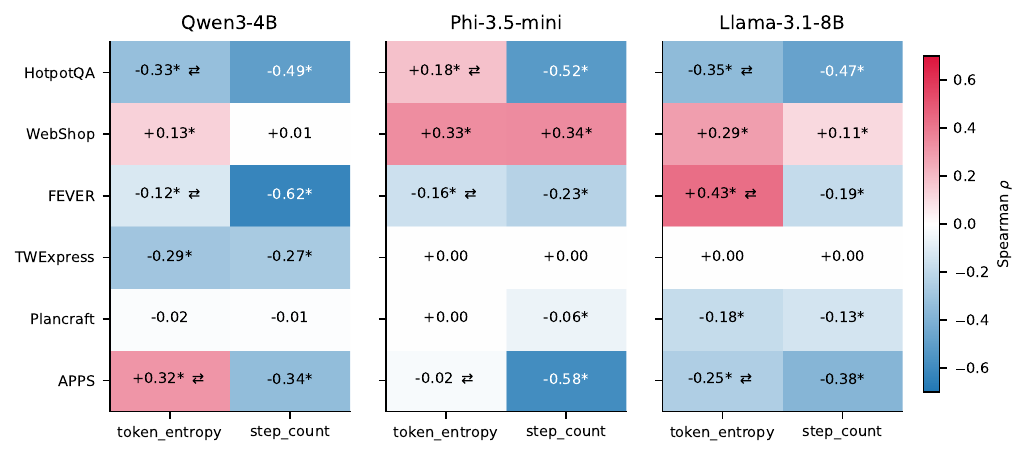}
	\caption{Multi-backbone signal direction heatmap. $^*$ marks
	cells significant at $p{<}0.05$. Cells marked with $\rightleftarrows$
	are \emph{sign-flip cells}: the same (environment, signal) pair
	yields a significantly positive $\rho$ on at least one backbone
	and a significantly negative $\rho$ on another, i.e.\ the
	signal--utility direction is reversed across model families.
	Entropy direction is model-dependent; structural signals
	(\texttt{step\_count}) are more stable.}
	\label{fig:multi-backbone-heatmap}
\end{figure*}
\begin{figure*}[h!]
	\centering
	\includegraphics[width=\textwidth]{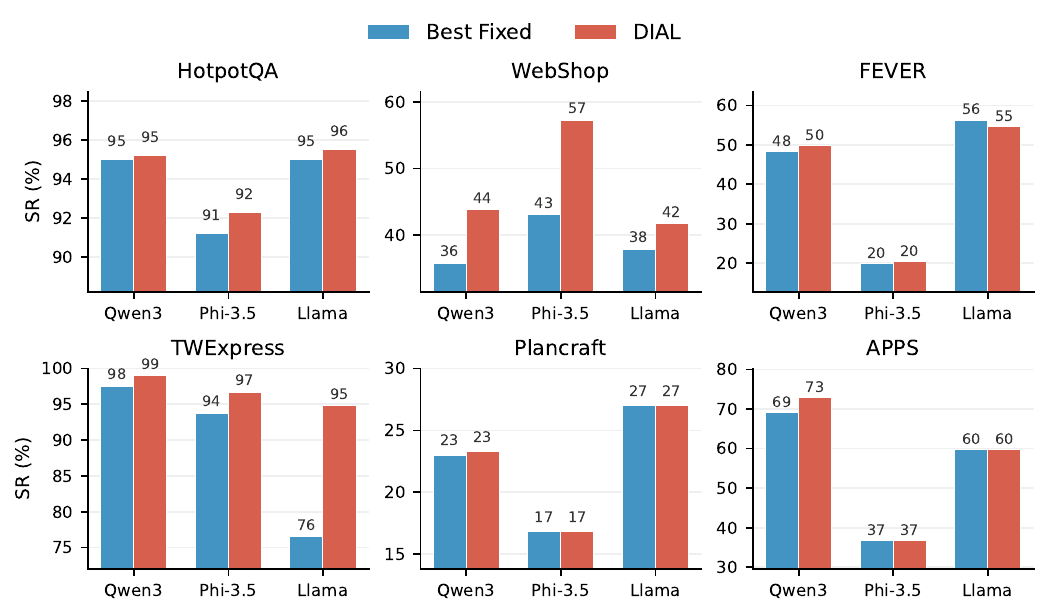}
	\caption{Cross-backbone comparison: DIAL (red) vs.\ best
		fixed-direction baseline (blue) across 3 backbones. Fixed baselines
		show large SR variance across backbones; DIAL adapts robustly.}
	\label{fig:multi-backbone-sr}
\end{figure*}

Figure~\ref{fig:multi-backbone-heatmap} visualizes this direction
variation, with black-outlined cells marking sign flips. The downstream
consequence appears in Figure~\ref{fig:multi-backbone-sr}: on
Phi-3.5 FEVER, most fixed-direction baselines score at or below
base\_only (7.2\% base $\to$ 8.5--23\% across baselines), while
DIAL adapts to the different sign and reaches 20.5\%. The
fixed-direction assumption is brittle when the model changes;
DIAL's per-(environment, backbone) direction discovery handles
this transparently.

\subsection{Environment-Adaptive Trigger Behavior}
\label{app:trigger-rate}

Beyond direction and feature selection, DIAL's gating magnitude
also adapts to each environment without explicit programming.
Figure~\ref{fig:trigger-adapt} shows the per-step trigger rate
across all 6 environments, ordered by rollout headroom~$\Delta$.
In rollout-safe TWExpress ($\Delta{=}{+}31.8$), DIAL triggers
aggressively (RR\,=\,73\%), consistent with the high headroom.
In rollout-harmful Plancraft ($\Delta{=}{-}7.0$), the trigger
rate decays from 49\% at step~0 to ${<}$20\% in later steps,
as the gate learns that late-step rollouts are increasingly
harmful. WebShop and APPS show moderate, stable trigger rates.
This full behavioral range emerges entirely from the
per-environment signed weights fitted in \S\ref{sec:learn}: DIAL
never receives an explicit headroom estimate, because signal
direction already encodes whether the environment rewards or
penalizes additional computation.

\begin{figure*}[h]
\centering
\includegraphics[width=\textwidth]{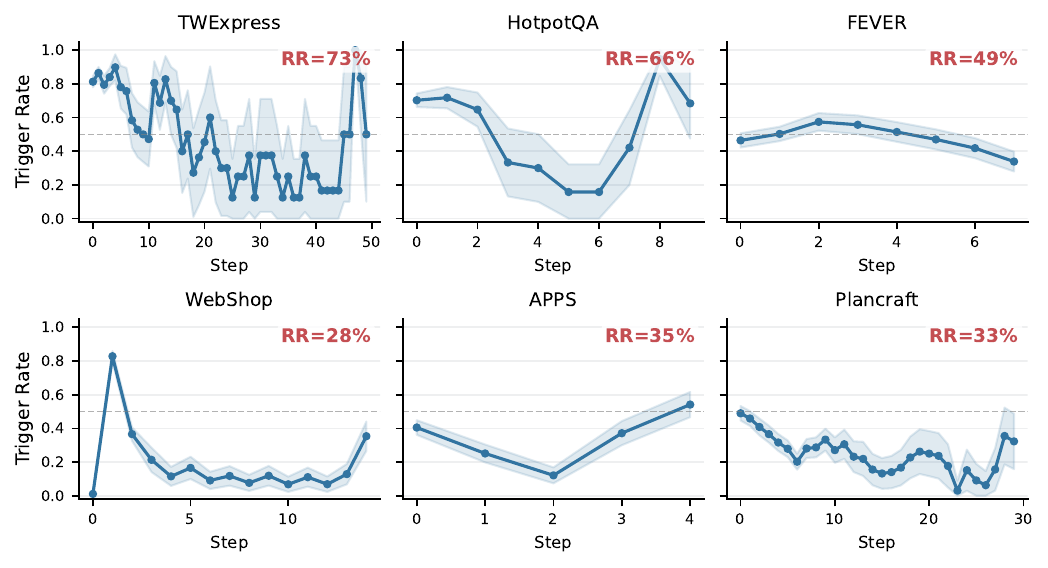}
\caption{Environment-adaptive trigger behavior. Panels ordered by
rollout headroom $\Delta$. Solid line: per-step trigger rate with
95\% CI band. Dashed gray: 50\% reference. RR: overall trigger
rate. DIAL's trigger rate ranges from 73\% (TWExpress) to $<$20\%
(late-step Plancraft), emerging from direction learning alone.}
\label{fig:trigger-adapt}
\end{figure*}

% ============================================================
% H. BROADER IMPACT
% ============================================================
\section{Broader Impact}
\label{app:impact}

This work addresses computational efficiency in LLM agent
deployment. By reducing unnecessary optimizer invocations, DIAL
lowers the energy and cost footprint of test-time compute without
sacrificing task performance. The prescriptive framework
(characterize information structure before designing gates) may
help practitioners avoid deploying fixed-direction methods that
actively harm performance in mismatched environments. We do not
foresee direct negative societal impacts beyond those inherent to
the LLM agents themselves.

% \newpage
% \input{checklist}

\end{document}